\def\year{2022}\relax
\documentclass[letterpaper]{article} 
\usepackage{aaai22}  
\usepackage{times}  
\usepackage{helvet}  
\usepackage{courier}  
\usepackage[hyphens]{url}  
\usepackage{graphicx} 
\usepackage{natbib}  
\usepackage{caption} 
\DeclareCaptionStyle{ruled}{labelfont=normalfont,labelsep=colon,strut=off} 
\usepackage{algorithm}
\usepackage[noend]{algorithmic}

\usepackage{newfloat}
\usepackage{listings}
\floatstyle{ruled}
\newfloat{listing}{tb}{lst}{}
\floatname{listing}{Listing}
\usepackage{algorithm}
\usepackage{algorithmic}
\usepackage{times}
\usepackage{soul}
\usepackage{url}
\usepackage[utf8]{inputenc}
\usepackage{amsmath}
\usepackage{amsthm}
\newtheorem{theorem}{Theorem}
\usepackage{graphicx,subfigure,caption}
\usepackage{caption}
\usepackage{comment}
\usepackage{enumitem}
\usepackage{amssymb}
\usepackage{xcolor}
\usepackage{tikz}
\usepackage{threeparttable}
\usepackage{enumitem}
\usepackage[most]{tcolorbox}
\usepackage{multirow}
\usepackage{pifont,amssymb}
\usepackage{lipsum,booktabs}
\usepackage[utf8]{inputenc}
\usepackage[english]{babel}
\usepackage[switch]{lineno}

\newtheorem{lemma}[theorem]{Lemma}

\title{Reverse Differentiation via Predictive Coding}
\author{
    Tommaso Salvatori\,\textsuperscript{\rm 1},
    Yuhang Song\,\textsuperscript{\rm 1,\,2\,}\thanks{Corresponding author.},
    Zhenghua Xu\,\textsuperscript{\rm 3},
    Thomas Lukasiewicz\,\textsuperscript{\rm 1},
    Rafal Bogacz\,\textsuperscript{\rm 2}
}
\affiliations{
    \textsuperscript{\rm 1} Department of Computer Science, University of Oxford, UK \\ 
    \textsuperscript{\rm 2} MRC Brain Network Dynamics Unit, University of Oxford, UK \\
    \textsuperscript{\rm 3} State Key Laboratory of Reliability and Intelligence of Electrical Equipment, \\Hebei University of Technology, Tianjin, China \\
    \{tommaso.salvatori, yuhang.song, thomas.lukasiewicz\}@cs.ox.ac.uk, rafal.bogacz@ndcn.ox.ac.uk, zhenghua.xu@hebut.edu.cn
}

\begin{document}

\maketitle

\begin{abstract}
Deep learning has redefined AI thanks to the rise of artificial neural networks, which are inspired by neuronal networks in the brain. Through the years, these interactions between AI and neuroscience have brought immense benefits to both fields, allowing neural networks to be used in a plethora of applications. 
Neural networks use an efficient implementation of reverse differentiation, called backpropagation (BP). This algorithm, however, is often criticized for its biological implausibility (e.g., lack of local update rules for the parameters).
Therefore, biologically plausible learning methods that rely on predictive coding (PC), a framework for describing information processing in the brain, are increasingly studied.
Recent works prove that these methods can approximate BP up to a certain margin on multilayer perceptrons (MLPs), and asymptotically on any other complex model, and that zero-divergence inference learning (Z-IL), a variant of PC, is able to exactly implement BP on MLPs.
However, the recent literature shows also that there is no biologically plausible method yet that can exactly replicate the weight update of BP on complex models.
To fill this gap, in this paper, we generalize (PC and) Z-IL by directly defining it on computational graphs, and show that it can perform exact reverse differentiation. What results is the first PC (and so biologically plausible) algorithm that is equivalent to BP in the way of updating parameters on any neural network, providing a bridge between the interdisciplinary research of neuroscience and deep learning.
Furthermore, the above results in particular also immediately provide a novel local and parallel implementation of BP. 
\end{abstract}

\section{Introduction}

In recent years, neural networks have achieved amazing results in multiple fields, such as image recognition~\cite{he2016deep,Krizhevsky2012}, natural language processing~\cite{Vaswani17,devlin-etal-2019-bert}, and game playing~\cite{silver17,Silver2016}. All the models designed to solve these problems share a common ancestor, multilayer perceptrons (MLPs), which are fully connected neural networks with a feedforward multilayer structure and a mapping function $\mathbb R ^n \rightarrow \mathbb R^m$. Although MLPs are able to approximate any continuous function~\cite{Hornik89} and theoretically can be used for any task, the empirical successes listed above show that more complex and task-oriented architectures perform significantly better than their fully connected ones. Hence, the last decades have seen the use of different layer structures, such as recurrent neural networks (RNNs)~\cite{Hochreiter97}, transformers~\cite{Vaswani17}, convolutional neural networks (CNNs), and residual neural networks~\cite{he2016deep}. Albeit diverse architectures may look completely different, their parameters are all trained using gradient-based methods, creating a need for a general framework to efficiently compute gradients. Computational graphs, which are decompositions of complex functions in elementary ones, represent the ideal solution for this task, as they generalize the concept of neural network. In fact, they allow the use of reverse differentiation to efficiently compute derivatives and hence update the parameters of the network. In deep learning, this technique is used to quickly propagate the output error through the network, and it is hence known under the name of \emph{error backpropagation} (\emph{BP})~\cite{rumelhart1986learning}. While being a milestone of the field, this algorithm has often been considered biologically implausible, as it does not follow the rules of biological networks in the brain to update the parameters and propagate information \cite{crick89}. Here, we use the term ``biologically plausible''  to refer to models that satisfy a list of minimal properties required by a possible neural implementation, namely, local computations and local plasticity (change in a connection weight depending only on the activity of the connected neurons) \cite{whittington2017approximation}.

An influential model of information processing in the brain, called \emph{predictive coding} (\emph{PC}) \cite{rao1999predictive}, is used to describe learning in the brain, and has promising theoretical interpretations, such as the minimization of free energy~\cite{bogacz2017tutorial,friston2003learning,friston2005theory,whittington2019theories} and probabilistic models~\cite{whittington2017approximation}. Originally proposed to solve unsupervised learning tasks, PC has been found to be successful also in supervised models~\cite{whittington2017approximation}, and its standard implementation, called \emph{inference learning} (\emph{IL})~\cite{whittington2017approximation}, has also been shown to be able to approximate asymptotically BP  on MLPs, and on any other complex model~\cite{millidge2020predictive}.
Furthermore, a~recent work has proved that PC can do exact BP on MLPs using a learning algorithm called \emph{zero-divergence inference learning} (\emph{Z-IL})~\cite{Song2020}. Z-IL is a biologically plausible method with local connections and local plasticity. While this exactness result is thrilling and promising, Z-IL has limited generality, as it has only been shown to hold for MLPs. 
Actually, a recent study shows that there are no published successful methods to train high-performing deep neural networks on difficult tasks (e.g., ImageNet classification) using any algorithm other than BP~\cite{Lillicrap20}. 
This shows the existence of a gap in our understanding of the biological plausibility of BP, which can be summarized as follows: there is an approximation result (IL), which has been shown to hold for any complex model~\cite{whittington2017approximation,millidge2020predictive}, and an exactness result (Z-IL), only proven for MLPs. 

In this work, we close this gap by analyzing the Z-IL algorithm, and generalize the exactness result to every complex neural network. Particularly, we start from analyzing the Z-IL algorithm on different architectures by performing one iteration of BP and one iteration of Z-IL on two identically initialized networks, and compare the two weight updates by computing the Euclidean distance. The results, reported in Table~\ref{tb:div} below, show two interesting things: first, they suggest that the exactness result holds for CNNs and many-to-one RNNs; second, they show that it does not hold for more complex architectures, such as residual and transformer neural networks. An analysis of the dynamics of the error propagation of Z-IL shows that the root of the problem is in the structure of the computational graph: in ResNet, for example, the skip connections design a pattern that does not allow Z-IL to exactly replicate the weight update of BP. In CNNs and RNNs, this does not happen. The main contributions of this paper are briefly summarized as follows. 
\begin{itemize}
    \item We show that Z-IL  is also able to exactly implement BP on CNNs and RNNs. Particularly, we give a direct derivation of the equations, and extend the proof of the original formulation of Z-IL on MLPs to CNNs and RNNs.
    
    \item We then generalize IL (and Z-IL) to work for every computational graph, and so any neural network. We also propose a variant of Z-IL that is directly defined on computational graphs, which we prove to be equivalent to BP in the way of updating parameters on any neural network.

\item  This results into a novel local  and parallel implementation of BP. We experimentally analyze the running time of Z-IL, IL, and BP on different architectures. The experiments show that Z-IL is comparable to BP in terms of efficiency, and several orders of magnitude faster than~IL. 

\item There are other impacts on machine learning beyond the above. 
In particular, the above novel formulation of BP in terms of IL may inspire 
other neuroscience-based alternatives to BP. 
Furthermore,  deep-learning-based approaches may actually be more closely related to information processing in the brain than commonly thought. 
    

\item At the same time, the first biologically plausible algorithm that exactly replicates the weight updates of BP on mapping functions of complex models may have a similarly big impact in neuroscience, as it shows that deep learning is actually highly relevant in neuroscience.
   \end{itemize}

\section{\label{sec:cg} Computational Graphs}

A \emph{computational graph} $G\,{=}\,(V,E)$, where $V$ is a finite non\-empty set of vertices, and $E$ is a finite set of edges, is a directed acyclic graph (DAG) that represents a complex function $\mathcal G$ as a composition of elementary functions. 
Every internal vertex $v_i$ is associated with one elementary function $g_i$, and represents the computational step expressed by $g_i$. Every edge pointing to this vertex represents an input of  $g_i$. 
For ease of presentation, the direction considered when using this notation is  the reverse pass (downwards arrows in Fig.~\ref{fig:cg}).
Furthermore, we call $e_{i,j} \in E$ the directed edge that starts at $v_i$ and ends at $v_j$.  
The first $n$ vertices $v_1,\dots, v_n$ are the leaves of the graph and represent the $n$ inputs of the function $\mathcal G$, while the last vertex $v^{\text{out}}$ represents the output of the function. We call $d_i$ the minimum distance from the output node $v^{\text{out}}$ to $v_i$ (i.e., the minimum number of edges separating $v_i$ and $v_{out}$). 
An example of a computational  graph for the function $\mathcal G(z_1,z_2) = (\sqrt{z_1} + z_2)^2$ is shown in Fig.~\ref{fig:cg}, where the arrows pointing upwards denote the forward pass, and the ones pointing downwards the reverse pass.
We  call $C(i)$ and $P(i)$ the indices of the child and parent vertices of~$v_i$, respectively. Hence, input nodes (nodes at the bottom) have no child vertices, and output nodes (nodes at the top) have no parent vertices.   We now briefly recall reverse differentiation, and so BP, on computational graphs, and we then newly define how to perform PC on computational~graphs.

\subsection{BP on Computational Graphs}

Let $\mathcal G\colon \mathbb R^n \rightarrow \mathbb R$ be a differentiable function, and $\{g_i\}$ be a factorization of $\mathcal G$ in elementary functions, which have to be computed according to a computational graph. Particularly, a computational graph $G=(V,E)$ associated with $\mathcal G$ is formed by a set of vertices $V$ with cardinality  $|V|$, and a set of directed edges $E$, where an edge $e_{i,j}$ is the arrow that points to $v_j$ starting from $v_i$. With every vertex $v_i \in V$, we associate an elementary function $g_i\colon\mathbb R^{k_i} \rightarrow \mathbb R$, where $k_i$ is the number of edges pointing to $v_i$. The choice of these functions is not unique, as there exist infinitely many ways of factoring $\mathcal G$. It hence defines the structure of a particular computational graph. Given an input vector $\bar z \in \mathbb R^n$, we denote by $\mu_i$  the value of the vertex $v_i$ during the forward pass. This value is computed iteratively as follows:
\begin{equation}
\mu_i = \begin{cases} 
z_i & \mbox{\!\!for } i \leq n\,;\\
g_i(\{\mu_j\}_{j \in C(i)}) & \mbox{\!\!for  } i > n\,.\\
\end{cases}
\label{eq:delta-recursive}
\end{equation}
We then have $\mathcal G(\bar z)= \mu_{|V|} = \mu_{out}$. The computational flow just described is represented by the upward arrows in Fig.~\ref{fig:cg}. We now introduce the classical problem of reverse differentiation, and show how it is used to compute the derivative relative to the output. Let $\bar z = (z_1, \dots, z_n)$ be an input (which in the case of MLPs will correspond to the weight parameters on the basis of which the output of the network is computed, as we will explain in the next section), and $\mathcal G(\bar z) \,{=}\, \mu_{out}$ be the output. Reverse differentiation is a key technique in machine learning and AI, as it allows to compute ${\partial \mathcal G}/{\partial z_i}$ for every $i<n$ efficiently. This is necessary to implement  BP at a reasonable computational cost, especially considering the extremely overparametrized architectures used today. This is done iteratively as follows: 
\begin{equation}
    \frac{\partial \mathcal G}{\partial \mu_i} = \sum\nolimits_{j \in P(i)} \frac{\partial \mathcal G}{\partial \mu_j} \cdot \frac{\partial \mu_j}{\partial \mu_i} = \sum\nolimits_{j \in P(i)} \frac{\partial \mathcal G}{\partial \mu_j} \cdot \frac{\partial g_j}{\partial \mu_i}.
\end{equation}
To obtain the desired formula for the input variables, it suffices to recall that $\mu_i = z_i$ for every $i \leq n$.

\smallskip\noindent
\textbf{Update of the leaf nodes: }Given an input $\bar z$, we consider a desired output $y$ for the function $\mathcal G$. The goal of a learning algorithm is to update the input parameters $(z_1, \dots, z_n)$ of a computational graph to minimize the quadratic loss $E = \frac 1 2 (\mu_{out} - y)^2$. Hence, the input parameters are updated by:
\begin{equation}
    \Delta z_i = - \alpha \cdot \frac{\partial  E}{\partial z_i} =  \alpha \cdot \sum\nolimits_{j \in P(i)} \delta_j \cdot\frac{\partial g_j}{\partial z_i},
    \label{eq:deltaz}
\end{equation}
where $\alpha$ is the learning rate, and ${\partial E}\,/\,{\partial z_i}$ is computed using reverse differentiation. We use the parameter $\delta_j$ to represent the error signal, i.e., the propagation of the output error among the vertices of the graph. It can be computed according to the following recursive formula:
\begin{equation}
\delta_{i} = \begin{cases} 
\mu_{out} - y & \mbox{if } i = |V|\,;\\
\sum\nolimits_{j \in P(i)} \delta_j \cdot\frac{\partial g_j}{\partial z_i} & \mbox{if } n<i<|V|.
\end{cases}
\label{eq:delta_err}
\end{equation}

\begin{figure}[t!]
\centering
	    \centering\vspace*{1ex}\hspace*{-2ex}\includegraphics[width=0.40\textwidth]{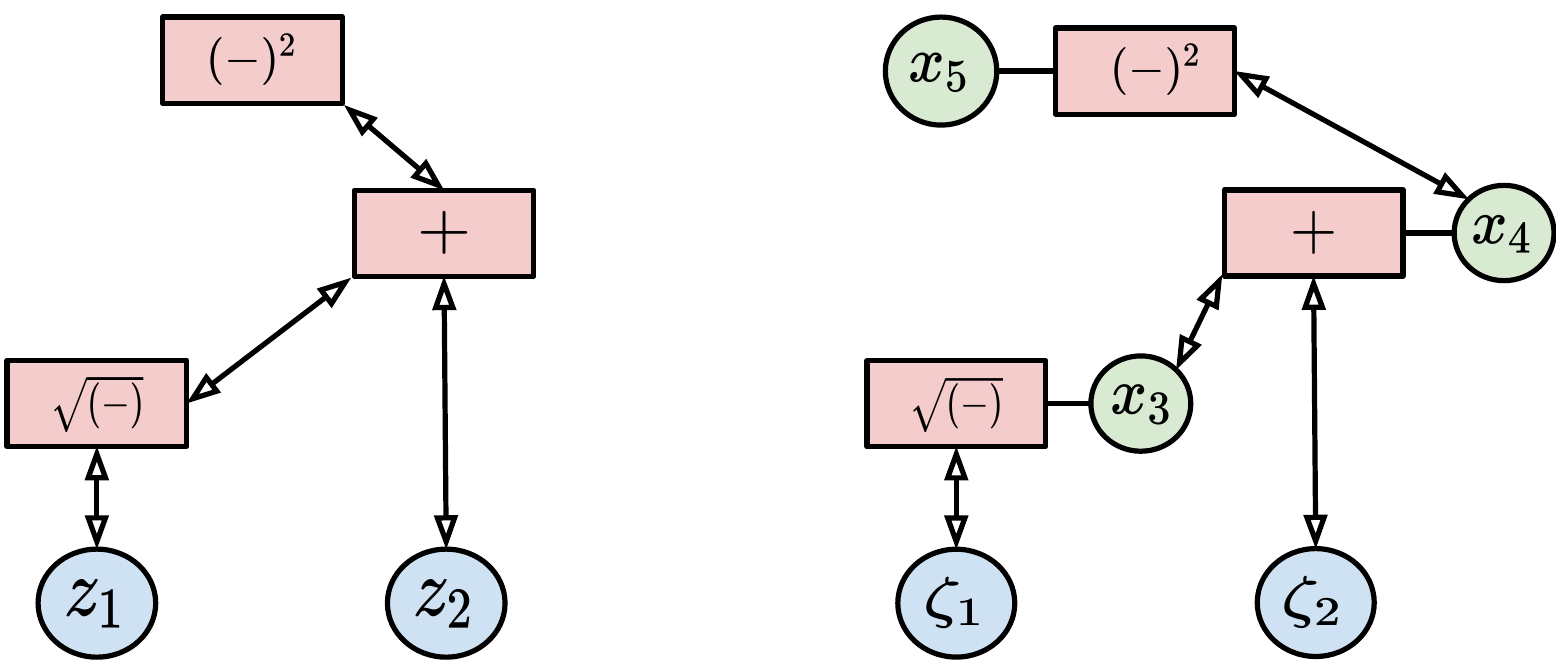}%
	\caption{Left: computational graph of the function $\mathcal G(z_1,$ $z_2)$ $=$ $(\sqrt{z_1} + z_2)^2$. Every internal vertex (red box) pictures its associated function $g_i$. Right: its predictive coding counterpart. The arrows pointing upwards are related to the feedforward pass.  The value nodes $x_i$ of the input neurons are set to the input of the function  ($\zeta_1$ and $\zeta_2$). Hence, we have omitted them from the plots to make the notation lighter. The same notation is adopted in later figures.}
	\label{fig:cg}\vspace*{-1ex}

\end{figure}

\subsection{IL on Computational Graphs}

\begin{figure*}[t]

	    \centering\vspace*{1ex}\includegraphics[width=0.75\textwidth]{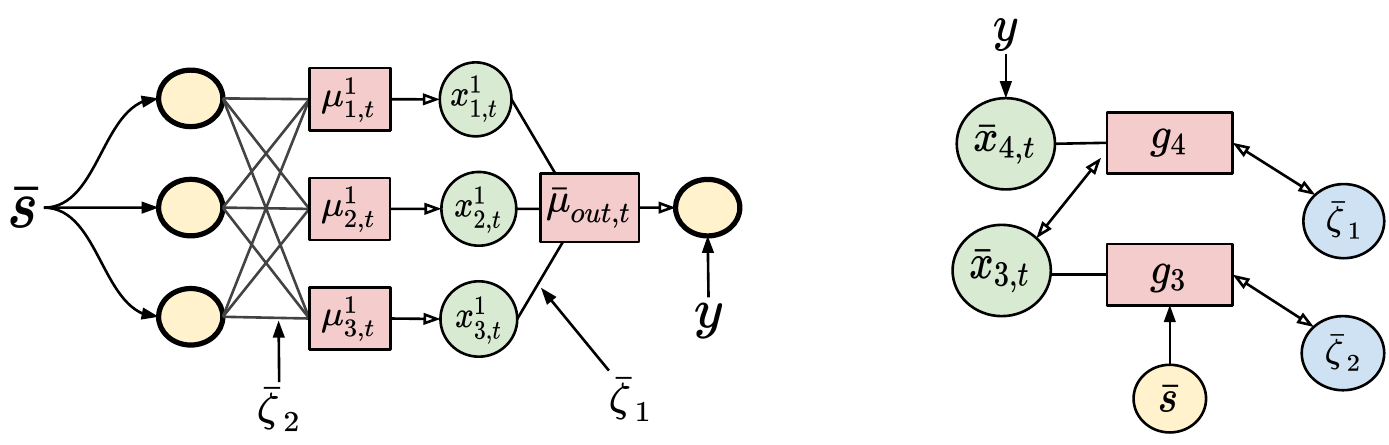}%
	\caption{Left: example of a 2-layer PCN. In these networks, it is possible to realize every computation locally using error nodes and value nodes in a biologically plausible way. For a more detailed discussion, we refer to \cite{whittington2017approximation}. Right: the corresponding computational graph.}
	\label{fig:zil-graph}\vspace*{-2.7ex}
\end{figure*}

We now show how the just introduced forward and backward passes change when considering a PC computational graph $G=(V,E)$ of the same function $\mathcal G$. A similar framework to the one that we are about to show has been developed in \cite{millidge2020predictive}. We associate with every vertex $v_i$, with $i>n$, a new time-dependent random variable $x_{i,t}$,  called \emph{value node}, and a prediction error $\varepsilon_{i,t}$. We denote a parameter vector (which for MLPs corresponds to weights) by $(\zeta_1, \dots, \zeta_n)$, so $\zeta_i$ in IL corresponds to $z_i$ in BP, but we use different symbols, as they may not be necessarily equal to each other. The values $\mu_i$ are computed as follows: for the leaf vertices, we have $\mu_{i,t} = \zeta_i$ and $\varepsilon_{i,t} = 0$  for $i \leq n$, while for the other values, we have
\begin{equation}
    \mu_{i,t} = g_i(\{x_{j,t}\}_{j \in C(i)}) \ \ \text{   and    } \ \ \varepsilon_{i,t} = \mu_{i,t} - x_{i,t}.
    \label{eq:mu-cg}
\end{equation}
This allows to compute the value $\mu_{i,t}$ of a vertex by only using information coming from vertices  connected to $v_i$. As in the case of PC networks, every computation is strictly local. The value nodes of the network are updated continuously  to minimize the following loss function, defined on all the vertices of $G$:
\begin{equation}{\small
F_t = \frac 1 2 \sum\nolimits_{i=1}^{|V|} (\varepsilon_{i,t})^2.
\label{eq:loss-cg}}
\end{equation}
The output $x_{out}$ of $\mathcal G(\bar \zeta)$ is then computed by minimizing this energy function through an inference process. The update rule is $\Delta x_{i,t} = - \gamma \,{\partial F_t}/{\partial x_{i,t}}$, where $\gamma$ is a small positive constant, called \emph{integration step}. Expanding this  gives:  
\begin{equation}
 \Delta x_{i,t} = - \gamma \cdot\frac{\partial F_t}{\partial x_{i,t}} = \gamma\cdot( \varepsilon_{i,t} - \sum\nolimits_{j \in P(i)} \varepsilon_{j,t} \cdot\frac{\partial \mu_{j,t}}{\partial x_{i,t}})\,.
 \label{eq:deltax-cg}
\end{equation}
Note that during the forward pass, all the value nodes $x_{i,t}$ converge to $\mu_i$, as $t$ grows to infinity. This makes the final output of the forward passes of  inference learning on the new computational graph equivalent to that of the normal computational graph.

\smallskip\noindent
\textbf{Update of the leaf nodes: } Let $ \bar \zeta$ be a parameter vector, and $y$ be a fixed target. To update the parameter vector and minimize the error on the output, we fix $x_{out} = y$. Thus, we have $\varepsilon_{out,t} = \mu_{out} - y$. By fixing the value node $x_{out,t}$, most of the error nodes can no longer decay to zero. Hence, the error $\varepsilon_{out,t}$ gets spread among the other error nodes on each vertex of the computational graph by running the inference process.  When the inference process has either converged, or it has run for a fixed number of iterations $T$, the parameter vector gets updated by minimizing the same loss function~$F_t$. Thus, we have: 
\begin{equation}
    \Delta \zeta_{i} = - \alpha \cdot\frac{\partial F_t}{\partial \zeta_{i}} = - \alpha \cdot\sum\nolimits_{j \in P(i)} \varepsilon_{j,t}\cdot\frac{\partial \mu_{j,t}}{\partial \zeta_{i}}.
    \label{eq:deltazeta}
\end{equation}
All computations are local (with local plasticity) in IL, and the model can autonomously switch between prediction and learning via running inference. The main difference between BP and IL on computational graphs is that the update of the parameters of BP is invariant of the structure of the computational graph: the way of decomposing the original function $\mathcal G$ into elementary functions does not affect the update of the parameters. This is not the case for IL, as different decompositions lead to different updates of the value nodes, and so of the parameters. However, it has been shown that, while following different dynamics, these updates are asymptotically equivalent \cite{millidge2020predictive}.

\section{Z-IL for MLPs}


\begin{algorithm}[t]
    \caption{Learning one training pair $({\bar s},y)$ with Z-IL}\label{algo:zil}
    \begin{algorithmic}[1]
    \REQUIRE $x_{out}$ is fixed to $y$; $\gamma=1$
    \STATE Initialize $x_{l,0} = \zeta_l$ for every leaf node; $x_{i,0} =\mu_{i,0}$ for every internal node
    \FOR{$t=0$ to $L$}
        \FOR {each vertex $v_i$}
            \STATE Update $x_{i,t}$ to minimize $F_{t}$ via Eq.~\eqref{eq:deltax-cg}
        \ENDFOR
        \IF{$t= l$}
            \STATE Update  $\bar \zeta_{l}$ to minimize $F_{t}$ via Eq.~\eqref{eq:deltazeta} \ \  
                {
                }
        \ENDIF
    \ENDFOR
    \end{algorithmic}
    \label{alg:zil}
\end{algorithm}

Recently, a new learning algorithm, called \emph{zero-diver\-gen\-ce inference learning} (\emph{Z-IL}), was shown to perform \emph{exact} backpropagation on fully connected predictive coding networks (PCNs), the PC equivalent of MLPs. Particularly, this result states that starting from a PCN and a MLP with the same parameters, the update of the weights after one iteration of BP is identical to the one given by one iteration of Z-IL.  We now provide a brief description of the original Z-IL algorithm. To be as close as possible to the original formulation of Z-IL, we adopt the same notation of that work, and index the layers starting from the output layer (layer $0$), and finishing at the input layer (layer $L$).

Let $\mathcal G(\bar z)$ be the function expressed by an artificial neural network (ANN), represented in Fig.~\ref{fig:zil-graph}. The leaf vertices  of its computational graph are the weight matrices, represented by the blue nodes in Fig.~\ref{fig:zil-graph}. Every weight matrix $\bar \zeta_l$ has the distance $l$ from the output vertex. 

This new algorithm differs from standard inference learning for the following reasons:

\begin{enumerate}
    \item The initial error $\varepsilon_{i,0}$ of every vertex $v_i$ is set to zero. This is done by performing a forward pass from an input vector $\bar s$ and setting $\mu_{i,0} = x_{i,0}$ for every vertex $v_i$.
    
    \item The weight parameters $\zeta_l$ of layer $l$ get only updated at time step  $t=l$, making the inference phase only last for~$L$ iterations. 
\end{enumerate}

\smallskip\noindent
\textbf{Update of the leaf nodes: } As stated, Z-IL introduces a new rule to update the weights of a fully connected PCN. Using the notation adopted for computational graphs, every leaf node~$\bar \zeta_l$ in Fig.~\ref{fig:zil-graph} gets updated at $t=l$. Alg.~\ref{alg:zil} shows how Z-IL performs a single update of the parameters when trained on a labelled point $(\bar s,y)$. For a detailed derivation of all the equations, we refer to the original paper \cite{Song2020}. The main theoretical result is as follows, formally stating that 
the update rules of BP and Z-IL are equivalent in MLPs.

\begin{theorem}\label{thm:1}
    \label{ther:final-equal} 
    Let $M$ be a fully connected PCN trained with Z-IL, and let $M'$ be its corresponding MLP, initialized as~$M$, and trained with BP. Then, given the same data point $s$ to both networks, we have 
   \begin{equation}
    \Delta \bar z_l = \Delta \bar \zeta_l 
    \end{equation}
    for every layer $l \geq 0$.
\end{theorem}

\section{Z-IL for CNNs and RNNs}

CNNs are a neural architecture that is highly used in computer vision, with a connectivity pattern that resembles the structure of animals' visual cortex. The parameters of a convolutional layer are contained in different \emph{kernels}, vectors that act on the input pattern via an operation called \emph{convolution}. Many-to-one RNNs, on the other hand, deal with sequential inputs, and consist of three different weight matrices: two are used recursively for the inputs and hidden layers, and the last one is the output layer.

While Theorem~\ref{thm:1} has only been proven for MLPs, the experimental results presented in Table~\ref{tb:div} suggest that the original formulation of Z-IL is also able to exactly replicate the weight update of BP on CNNs and RNNs. Inspired by our empirical findings, we  prove that the update rules of BP and Z-IL are equivalent in convolutional and recurrent networks, generalizing the result of Theorem~\ref{thm:1} to CNNs and RNNs: 

\begin{theorem}\label{thm:2}
    \label{ther:final-equal-2} 
    Let $M$ be a convolutional or a recurrent PCN trained with Z-IL, and let $M'$ be its corresponding model, initialized as~$M$, and trained with BP. Then, given the same data point to both networks, the update of all parameters performed by Z-IL on $M$ is equivalent to that of BP on $M'$.
\end{theorem}

The experimental results presented in Table~\ref{tb:div}, however, show that the original definition of Z-IL does \emph{not} generalize to more complex architectures. In what follows, we solve this problem by defining Z-IL directly on computational graphs, and prove a generalization of Theorems~\ref{thm:1}~and~\ref{thm:2}.


\begin{figure}[t]
\centering
\includegraphics[width=0.47\textwidth]{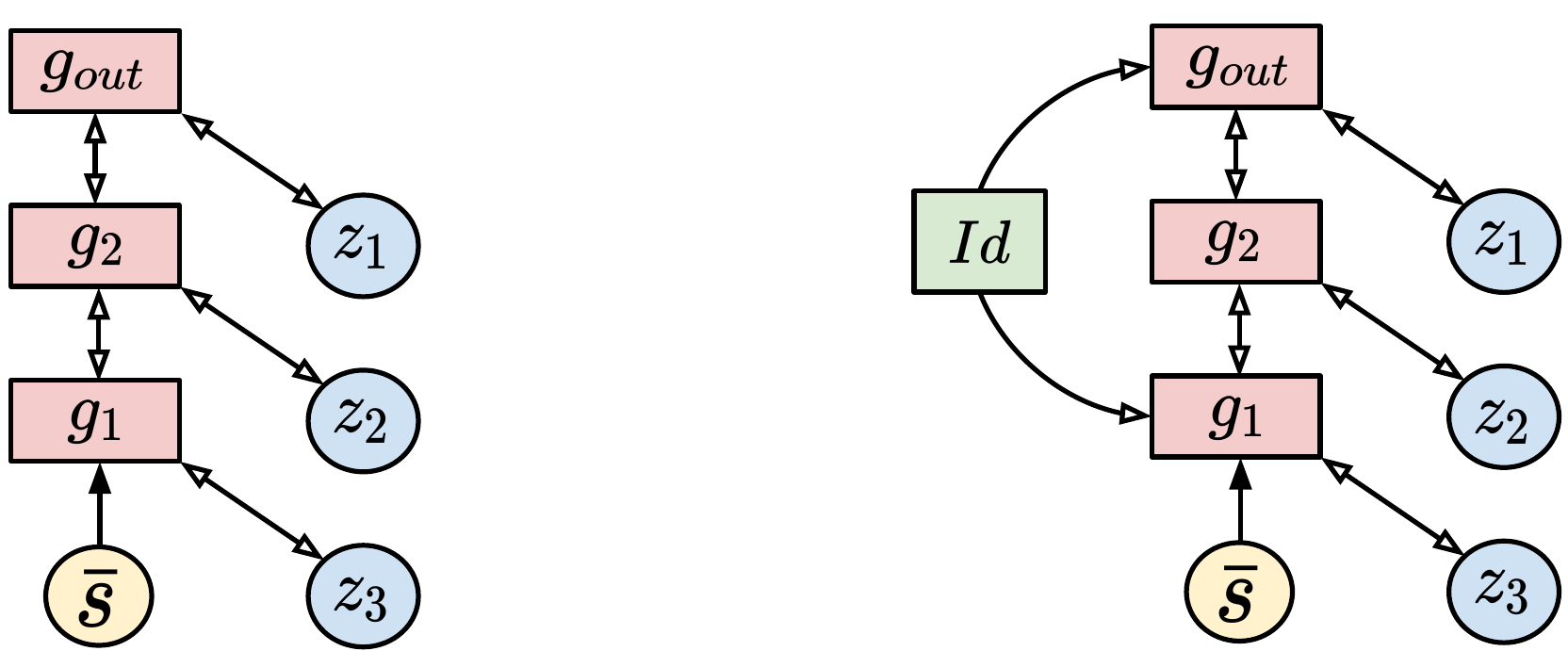}

	\caption{Left: computational graph of a 3-layer MLP with a residual connection, corresponding to the function ${\mathcal G}(s,$ $\bar z)= s z_3 + s z_3 z_2 z_1$. Right: an equivalent computational graph, with the addiction of an identity node.}
	\label{fig:res}
\end{figure}

\begin{table}[!t]\vspace*{-0.70ex}
  \centering
\resizebox{1.0\columnwidth}{!}{
    \begin{tabular}{@{}cccccc@{}}
    \toprule
    \cmidrule(r){1-4}
     & MLP & CNNs  & RNNs & ResNet18 & Transformer            \\
    \midrule
    Divergence:      & 0 & 0 & 0 & $4.53 \times 10^7$  & $7.29 \times 10^4$ \\
    \bottomrule
    \end{tabular}
  }\vspace*{-1.0ex}
\caption{Divergence between one update of weights of BP and Z-IL on different models, initialized in the same way.}
\label{tb:div}
\end{table}

\section{The Problem of Skip Connections}

In this section, we provide a toy example that shows how  Z-IL and BP behave on the computational graph of an ANN with a skip connection. Particularly, we show that it is impossible for Z-IL to replicate the same update of BP on all the parameters, unless the structure of the computational graph is altered. Consider the following function, corresponding to a simple MLP with a skip connection, represented in Fig.~\ref{fig:res}, left side:
\begin{equation}
    \mathcal G(s, \bar z) = s z_3 + s z_3 z_2 z_1.
\end{equation} 

\smallskip\noindent
\textbf{BP: } Given an input value $s$ and a desired target $y$, BP computes the gradient of every leaf node using reverse differentiation, and updates the parameters of $z_3$ as follows:
\begin{equation}
    \Delta z_3 = - \alpha \cdot \frac{\partial E}{\partial z_3} = \alpha \cdot \delta(z_1 z_2+1) s,
    \label{eq:update}
\end{equation}
where $\delta = (\mu_{out}-y)$, and $E$ is the quadratic loss defined on the output node. 

\smallskip\noindent
\textbf{Z-IL: } Given an input value $s$ and a desired target $y$, the inference phase propagates the output error through the graph via Eq.~\eqref{eq:deltazeta}. Z-IL updates $\zeta_3$ at $t=3$, as it belongs to the third hidden layer. This leads to the following:
\begin{equation}
    \Delta \zeta_3 = - \alpha \cdot \frac{\partial F_3}{\partial \zeta_3} = \alpha \cdot \delta \zeta_1 \zeta_2 s,
\end{equation}
where $\delta = \varepsilon_{out,0} =  (\mu_{out,0}-y)$, and $F_2$ is computed according to Eq.~\eqref{eq:loss-cg}. Note that this update is different from the one obtained by BP. We now analyze the reason of this mismatch and provide a solution.

\subsection{Identity Vertices}

The error signal propagated by the inference reaches $\zeta_3$ in two different moments: $t=2$ from the output vertex, and $t=3$ from $g_2$. Dealing with vertices that receive error signals in different moments is problematic for the original formulation of the Z-IL algorithm, as every leaf node only gets updated once. Furthermore, changing the update rule of  Z-IL does not solve the problem, as no other combination of updates produces the same weight update defined in Eq.~\eqref{eq:update}. To solve this problem, we then have to assure that every node of the graph is reached by the error signal in a single time step. This result is trivially obtained on computational graphs that are levelled DAGs, i.e., graphs where every directed path connecting two vertices has the same length. Here, the error reaches every vertex at a single, specific time step, no matter how complex the graph structure is. We now show how to make every computational graph levelled, without affecting the underlying function and the computations of the derivatives. 

Every elementary function $g_i$ can be written as a composition with the identity function, i.e., $g_i \circ Id$. Given two vertices $v_i$ and $v_j$ connected via the edge $e_{i,j}$, it is then possible to add a new vertex $v_k$  by splitting the edge $e_{i,j}$ into $e_{i,k}$ and $e_{k,j}$, whose associated function $g_k$ is the identity. This leaves the function expressed by the computational graph unvaried, as well as the computation of the derivatives, the forward pass, and the backward pass of BP. However, placing the identity vertices in the correct places, makes the computational graph levelled, allowing every vertex to receive the error signals at the same time step. Consider now the levelled graph of Fig.~\ref{fig:res}, right side, where an identity node has been added in the skip connection. The error signal of both~$g_1$ and $g_{out}$ reaches $g_2$ simultaneously at $t=2$. Hence, at $t=3$, Z-IL updates $\zeta_3$ as follows:
\begin{equation}
    \Delta \zeta_3 = - \alpha \cdot \frac{\partial F_3}{\partial \zeta_3} = \alpha \cdot \delta( \zeta_1 \zeta_2 + 1) s.
\end{equation}
If we have $\zeta_i = z_i$, this weight update is equivalent to the one performed by BP and expressed in Eq.~\eqref{eq:update}. Hence, Z-IL is able to produce the same weight update of BP in a simple neural network with one skip connection, thanks to a single identity vertex. In the next section, we generalize this result.

\section{Levelled Computational Graphs}

\begin{figure}[t]
\centering
\includegraphics[width=\linewidth]{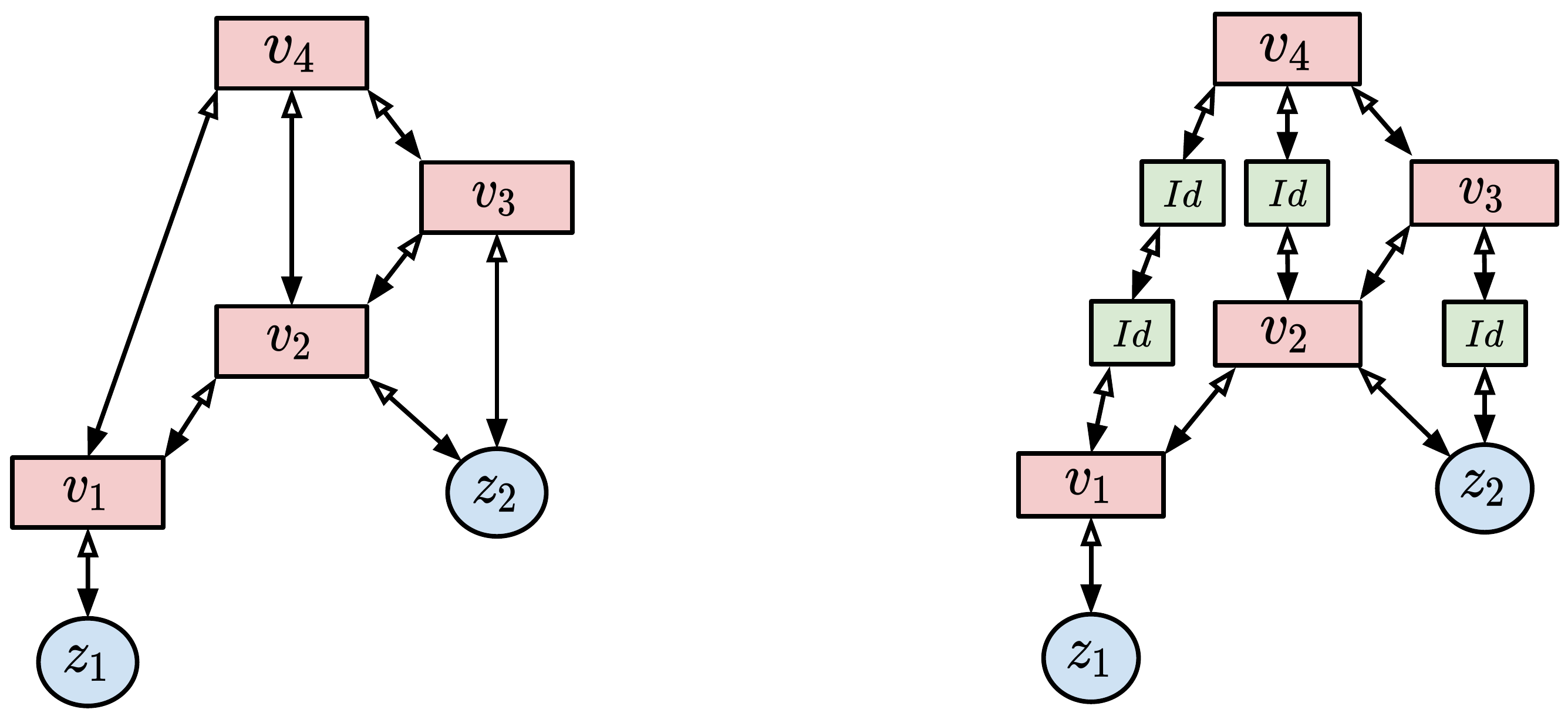}\vspace*{-1ex}

	\caption{Computational graphs of the same function $\mathcal G$. Left: the original graph $G$. Right: the transformed graph, with the identity vertices in green.}
	\label{fig:id}\vspace*{-1ex}
\end{figure}

In this section, we show that, given any computational graph, it is always possible to generate an equivalent, levelled version of it. Particularly, we provide an algorithm that performs this task by adding identity nodes. This leads to the first result needed to prove our main theorem: given any function $\mathcal G$, it is always possible to consider an equivalent, levelled, computational graph. This allows to partition the nodes of $G$ in a \emph{level structure}, where a level structure of a directed graph is a partition of the vertices into subsets that have the same distance from the top vertex.

Let $G$ be a computational graph, and $S_1,\dots,S_K$ be the family of subsets of $V$ defined as follows: a vertex $v_i$ is contained in $S_k$ if there exists a directed path of length $k$ connecting $v_i$ to $v_{out}$, i.e., 
\begin{equation}
    S_k = \{ v_i \in V | \ \exists \ \text{a path} \ (e_{out,j_1}, \dots, e_{j_{k-1},i})\}\,.
    \label{eq:sub}
\end{equation}
Hence, we have that $v_{out}$ is contained in $S_0$, its children vertices in $S_1$, and so on. In a levelled graph, every vertex is contained in one and only one of the subsets, and this partition defines its level structure. Let $D_i$ be the maximum distance between $v_{out}$ and the parent nodes of $v_i$, i.e., $D_i = \max_{v_j \in P(i)} d_j$. We now show for every DAG $G$ how to make every vertex $v_i$ to be contained in only one subset $S_k$, without altering the dynamics of the computational graph via the addition of identity nodes. 

Let $G$ be a DAG with root $v_0$, and let $(v_0,v_1,\dots,v_n)$ be a topological sort of the vertices of $G$. Starting from the root, for every vertex $v_j$, we replace every existing edge $e_{i,j}$ with the following path:
\begin{equation}
    v_i \rightarrow Id \rightarrow \dots \rightarrow Id \rightarrow v_j,
\end{equation}
which connects $v_i$ to $v_j$ via $d_j - D_i$ identity nodes. When this process has been repeated on all the vertices, we obtain a levelled DAG. This is equivalent to having every $v_i \in G$ that belongs to one and only one subset $S_k$, as every pair of disconnected paths between two vertices has the same length, thanks to the addition of identity vertices. Hence: 
\begin{theorem}
Given a function $\mathcal G : \mathbb R^n \rightarrow \mathbb R$ and any factorization of it expressed by elementary functions $\{g_i\}$, there exist a levelled computational graph $G=(V,E)$ that represents this factorization.
\end{theorem}

The above theorem shows that every neural network can be expressed as a levelled computational graph, and hence that every result shown for levelled computational graphs can be naturally extended to every possible neural network. 
\begin{algorithm}[t]
    \caption{Generating a levelled DAG $G'$ from $G$}\label{algo:level}
    \begin{algorithmic}[1]
    \REQUIRE $G$ is a DAG, and $(v_0,\dots,v_n)$ a topological sort.
        \FOR{every $j$ in $(0,n)$ included}
            \FOR{each vertex $v_i$ in $P(j)$}
                \STATE Add $(d_j - D_i)$ identity vertices to $e_{i,j}$
            \ENDFOR
        \ENDFOR
    \end{algorithmic}
    \label{algo:levelled}
\end{algorithm}

\section{Z-IL for Levelled Computational Graphs}\label{sec:zil}

In this section, we show that a generalized version of Z-IL allows PCNs to do exact BP on any computational graph.

Let $G\,{=}\,(V,E)$ be the levelled computational graph of a function $\mathcal G : \mathbb R^n \rightarrow \mathbb R$, and consider the partition of $V$ via its level structure $S_1,\dots,S_K$. We now present a variation of IL for computational graphs that allows predictive coding to exactly replicate the parameter update of BP, called Z-IL for computational graphs. This algorithm is similar to IL, but the following two differences are introduced:

\smallskip\noindent
\textbf{Forward pass:} Differently from IL, where input and output are presented simultaneously, Z-IL first presents the input vector to the function, and performs a forward pass. Then, once the values $\mu_{i}$ of all the internal vertices have been computed, the value nodes are initialized to have zero error, i.e., $ x_{i,0} \,{=}\, \mu_{i}$, and the output node is set equal to the label $y$. This is done to emulate the behaviour of BP, which first computes the output vector, and then compares it to the label.

\smallskip\noindent
\textbf{Update of the leaf nodes:} Instead of continuously running inference on all the leaf nodes of $G$, we only run it on the internal vertices. Then, at every time step $t$, we update all the leaf nodes  $v_i \in S_t$, if any. More formally, for every internal vertex $v_i$, training continues as usual via Eq.~\eqref{eq:deltax-cg}, while leaf nodes are updated according to the following equation:
\begin{equation}
    \!\Delta \zeta_{i,t}
    = \begin{cases}
    \gamma \cdot \sum_{j \in P(i)} \varepsilon_{j,t}\cdot\frac{\partial \mu_j}{\partial \zeta_i} &  \!\!\mbox{if } v_i \in S_t \\
    0  &  \!\!\mbox{if } v_i \not\in S_t.
    \end{cases}
    \label{eq:zeta}
\end{equation}%

This shows that one full update of the parameters requires $t=K$ steps. Note that for multilayer networks, $K$ is equal to the number of layers $L$. Overall, the functioning of Z-IL for computational graphs is summarized in Algorithm~\ref{algo:cg-ZIL}. We now show that this new formulation of Z-IL is able to replicate the same weight update of BP on any function $\mathcal G$.
    
\begin{table*}[t]

  \vspace*{0.75ex}
  \centering
  \resizebox{1.0\textwidth}{!}{
    \begin{tabular}{@{}cccccc@{}}
    \toprule
    \cmidrule(r){1-5}
    Method  & MLP & AlexNet~\cite{Krizhevsky2012} & RNN & ResNet18~\cite{he2016deep} & Transformer~\cite{Vaswani17}   \\
    \midrule
    BP        & $3.72$    & $8.61$     & $5.64$ & $12.43$ & $20.43$  \\
    IL        & $594.25$  & $661.53$  & $420.01$ & $1452.34$ & $1842.64$  \\
    Z-IL        & $3.81$    & $8.86$     & $5.67$ & $12.53$ & $20.53$  \\

    \bottomrule
    \end{tabular}
  }\vspace*{-1.5ex}
\caption{Average running time of each weights update (in ms) of BP, IL, and Z-IL for computational graphs.}
\label{tb:time}
\end{table*}

\begin{algorithm}[t]
    \caption{Z-IL for computational graphs.}\label{algo:cg-ZIL}
    \begin{algorithmic}[1]
    \REQUIRE $x_{out}$ is fixed to a label $y$, 
    \REQUIRE $\{S_k\}_{k=0, \dots, K}$ is a level structure of $G(V,E)$;
    \REQUIRE $x_{i,0} =\mu_{i,0}$ for every internal node.
    \FOR{$t=0$ to $K$}
        \STATE Update each $x_{i,t}$ to minimize $F_{t}$ via Eq.~\eqref{eq:deltax-cg}
        \STATE Update each $\zeta_{i,t} \in S_t$ to minimize $F_{t}$ via Eq.~\eqref{eq:deltazeta} \ \  
    \ENDFOR
    \end{algorithmic}
\end{algorithm}

\begin{theorem}
Let $(\bar z,y)$ and $(\bar \zeta,y)$  be two points with the same label $y$, and $\mathcal G: \mathbb R^n \rightarrow \mathbb R$ be a function. Assume that the update $\Delta \bar z$ is computed using BP, and the update $\Delta \bar \zeta$ using Z-IL with $\gamma = 1$. 
Then, if $\bar z = \bar \zeta$, and we consider a levelled computational graph of $\mathcal G$, we have 
\begin{equation}
    \Delta z_i = \Delta \zeta_i
\end{equation}
for every $i \leq n$.
\end{theorem}

This proves the main claims made about Z-IL: (i) exact BP and exact reverse differentiation can be made biologically plausible on the computational graph of \emph{any} function, and (ii)~Z-IL is a learning algorithm that allows PCNs to perfectly replicate the dynamics of BP on any function. Particularly, adding identity nodes to the computational graphs to produce equivalence to BP has non-trivial implications: it shows that the key difference between the PC model of learning in the brain and BP  lies in the synchronization of error propagation. This offers a novel perspective to investigate the gap between BP and neural models.

\section{Experiments}

In the above sections, we have theoretically proved that the proposed generalized version of Z-IL is equivalent to BP on every possible neural model. Multiple experiments, reported in the supplementary material, further confirmed this: the divergences of weight updating between BP and Z-IL are always zero on all tested neural networks.
So, there is no need for detailed experimental evaluation for the equivalence. In this section, we will complete the picture of this work with experimental studies to evaluate the computational efficiency of Z-IL, and quantitatively compare it with those of BP and IL. Particularly, we perform extensive experiments on different architectures, testing multiple models per architecture. The results of BP, IL, and Z-IL, averaged over all the experiments per model, are reported in Table~\ref{tb:time}, and a detailed description of the experiments, as well as all the parameters needed to reproduce the results, are provided in the supplementary material.

\subsection{Results and Evaluations}
As shown in Table~\ref{tb:time}, the computational time of Z-IL is very close to that of BP, and orders of magnitude lower than that of IL. This proves that Z-IL is an efficient alternative to BP in practice, instead of just being a theoretical tool. The high computational time of IL is due to the large number of iterations $T$. For example, for small MLPs, $T$ is set to $20$ in~\cite{whittington2017approximation}, and as larger models require higher numbers of iterations to converge, $T$ is set between $100$ and $200$ for mid-size architectures, such as RNNs and CNNs in~\cite{millidge2020predictive}.
Note that the approximation results of these works are achieved with fixed values of $T$, and not at convergence. Z-IL explains the above findings, as we show that strict equivalence can be achieved with a small number of inference steps; one just needs to satisfy the proposed conditions properly.

\section{Related Work}

PC is an influential theory of cortical function in theoretical and computational neuroscience, as it provides a computational framework, able to describe information processing in multiple brain areas \cite{friston2005theory}.
It has appealing theoretical interpretations, such as free-energy minimization~\cite{bogacz2017tutorial,friston2003learning,friston2005theory} and variational inference of  probabilistic models~\cite{whittington2017approximation}.
There are also variants of PC developed into different biologically plausible process theories specifying cortical microcircuits that potentially implement such theories~\cite{bastos2012canonical,kanai2015cerebral,shipp2016neural}.
Moreover, the central role of top-down predictions is consistent with the ubiquity and importance of top-down diffuse connections between cortical areas. 
PC is then consistent with many known aspects of neurophysiology, and has been translated into biologically plausible process theories which specify potential cortical microcircuits which could implement the algorithm.
Due to this solid biological grounding, PC is also attracting interest in machine learning recently, especially focusing on finding the links between PC and BP \cite{whittington2017approximation}.

{Biologically plausible approximations to BP} have been intensively studied, because on the one hand, the underlying principles of BP are unrealistic for an implementation in the brain~\cite{crick89, lillicrap2016random,Lillicrap20}, but on the other hand, BP outperforms all alternative discovered frameworks~\cite{baldi2016theory}. Bridging the gaps between BP and learning in biological neuronal networks of the brain (learning in the brain, for short, or simply BL) has been a major open question for both neuroscience and machine learning. \cite{whittington2019theories,rao2020backpropagation,kriegeskorte2015deep,kietzmann2018deep,richards2019deep}.
However, earlier biologically plausible approximations to BP have not been shown to scale to complex problems, such as learning colored images \cite{lillicrap2016random,o1996biologically,kording2001supervised,bengio2014auto,lee2015difference,nokland2016direct,scellier2017equilibrium,scellier2018generalization,lin2018dictionary,illing2019biologically}.
More recent works show the capacity of scaling up biologically plausible approximations to the level of BP \cite{xiao2018biologically,obeid2019structured,nokland2019training,amit2019deep,aljadeff2019cortical,akrout2019using,wang2020supervised}. 
However, to date, none of the earlier or recent models 
has bridged the gaps at a degree of demonstrating an equivalence to BP, though some of them \cite{lee2015difference,whittington2017approximation,nokland2019training,ororbia2017learning,millidge2020predictive} demonstrate that they  approximate BP, or are equivalent to BP  under unrealistic restrictions \cite{xie2003equivalence,sacramento2018dendritic}.

\section{Summary and Outlook}

The gap between machine learning and neuroscience is currently opening up: on the one hand, recent neural architectures trained by BP are invented with impressive performance in machine learning; on the other hand, models in neuroscience can only match the performance of BP in small-scale problems. There is thus a crucial open question of whether the advanced architectures in machine learning are actually relevant for neuroscientists. In this paper, we show that all these advanced architectures can be trained with one of their neural models: the proposed generalization of Z-IL is always equivalent to BP, with no extra restriction on the mapping function and the type of neural networks. (Previous works  only showed that IL approximates BP in single-step weight updates under unrealistic and non-trivial requirements.) 
Also, the computational efficiency of Z-IL is comparable to that of BP, and is several orders of magnitude better than IL. 
 Hence, we obtain a novel local  and parallel implementation of BP. Moreover, the novel formulation of BP in terms of IL may inspire 
other neuroscience-based alternatives to BP. The exploration of such alternatives to BP are a topic of our ongoing research. 
Furthermore, our results show that deep-learning-based models may actually be more closely related to information processing in the brain than commonly thought, which may have a big impact on both the machine learning and the neuroscience community. 


\section{Acknowledgments}

This work was supported by the Alan Turing Institute under the EPSRC grant EP/N510129/1, by the AXA Research Fund, by the EPSRC grant EP/R013667/1, and by the EU TAILOR grant. We also acknowledge the use of the EPSRC-funded Tier 2 facility JADE (EP/P020275/1) and GPU computing support by Scan Computers International Ltd.
This work was also supported by the China Scholarship Council under the State Scholarship Fund, by J.P.\ Morgan AI Research Awards, by the UK Medical Research Council under the grant MC\_UU\_00003/1, by the National Natural Science Foundation of China under the grant 61906063, by the Natural Science Foundation of Hebei Province, China, under the grant F2021202064, by the Natural Science Foundation of Tianjin City, China, under the grant 19JCQNJC00400, and by the ``100 Talents Plan'' of Hebei Province, China, under the grant E2019050017.

\bibliography{references}

\newpage
\appendix

\newpage

\newpage

\begin{figure*}[h]

	    \centering\vspace*{1ex}\includegraphics[width=1.0\textwidth]{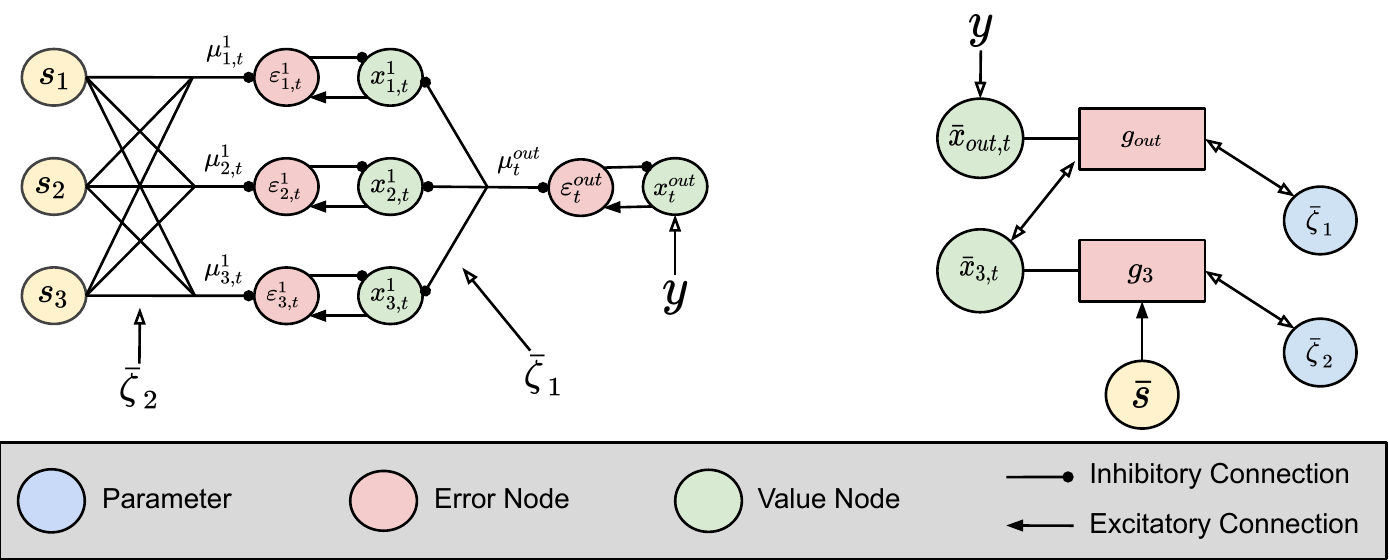}%
	\caption{Left: example of a 2-layer PCN with inhibitory and excitatory connections. In these networks, it is possible to realize every computation locally using error nodes and value nodes in a biologically plausible way. For a more detailed discussion, we refer to \cite{whittington2017approximation}. Right: the corresponding computational graph.}
	\label{fig:bio}

\end{figure*}

\section{Biological Plausibility of PCNs}

The term \emph{biologically plausible} has been extensively used in the computational neuroscience literature, often with different meanings. As mentioned in the introduction, in this paper, biological plausibility concerns a list of minimal properties that a learning rule should satisfy to have a possible neural implementation: computations should be local (i.e., each neuron adjusts its activity just based on input that it receives from connected neurons), and plasticity should also be local (i.e., the change in a
connection weight depends only on the activity of the connected neurons). BP is biologically implausible mainly due to the fact that it lacks locality of plasticity. In BP, the change in each synaptic weight during learning is calculated as a global function of activities and weights of many neurons (often not connected with the synapse being modified). In the brain, however, each neuron in the network must perform its learning algorithm locally, without external influence, and the change in each synaptic weight must depend on just the activity of the neurons connected via this synapse.

To improve the clarity of the presentation, we have decided to describe PCNs using value nodes $\bar x^l_{i,t}$ and their predictions $\bar \mu^l_{i,t}$. This presentation, however, does not fully highlight the reasons that make PCNs trained with IL and Z-IL biologically plausible. We now address this problem. It is in fact possible to represent PCNs using only local information, which gets propagated through the network via inhibitory and excitatory connections. Particularly, the value nodes of a layer $l$ are connected to the error nodes of layer $l-1$ via inhibitory connections. The same holds for computational graphs. Graphical representations of a 2-layer PCN and its computational graph are given in Fig.~\ref{fig:bio}.

Regarding the identity nodes, they are a ``trick'' that we have implemented to delay the signal between two neurons. However, there is an equivalent concept in neurobiology: it is known that different connections between neurons have different transmission delays, and this principle is widely used by the brain in its computations (e.g., it has been suggested that to detect movement in particular directions, visual neurons receive inputs from different locations on the retina with different delays).
So, the most plausible way of interpreting the identity nodes is not as physical neurons, but as transmission delays between connected neurons. Regarding the design, this delay mostly depends on the length of the dendrite between the synapse and the cell body. As presynaptic neurons can produce synapses on multiple locations of the dendrite, it is possible to select the location of the delayed signal by strengthening these particular synapses/delays through a learning~process.

\section{Empirical Validation of the Theorems}\label{sec:val}

\begin{table*}[ht]
  \caption{Euclidean distance of the weights after one training step of Z-IL (and variations), and BP.}
  
  \medskip 
   \vspace*{-1.5ex}
  \centering
  \resizebox{0.9\textwidth}{!}{
    \begin{tabular}{@{}ccccc@{}}
    \toprule
    \cmidrule(r){1-4}
    Model & Z-IL  & Z-IL without Level-dependent Update  & Z-IL with $\varepsilon^l_{i,0} \neq 0$ & Z-IL with $\gamma = 0.5$ \\
    \midrule
    MLP     & $0$ & $1.42\times10^2 $   & $7.22$ & $8.67 \times 10^4$ \\
    RNN     & $0$ & $6.05\times10^3$   & $9.60$ & $6.91\times10^5$ \\
    CNN     & $0$ & $5.93\times10^5$   & $7.93\times10^2$ & $9.87\times10^8 $ \\
    ResNet     &  $0$ & $9.43\times10^7$ & $4.53\times10^5$ & $6.44\times10^9$ \\
    Transformer     &  $0$ & $1.12\times10^{11}$ & $3.41\times10^6$ & $8.63\times10^{16}$ 
    \\
    \bottomrule
    \end{tabular}
  }
  \label{tb:abl}
\end{table*}

To empirically validate the results of our theorems, we show that all the conditions of Z-IL are needed to obtain exact backpropagation. Particularly, by starting from the same weight initialization, we have conducted one training step of the following five different learning algorithms: (i) BP, (ii)~Z-IL, (iii) Z-IL without level-dependent update, (iv) Z-IL with $\varepsilon^l_{i,0} \neq 0$, and (v) Z-IL with $\gamma = 0.5$.
%
%
Note that the last three algorithms are variations of Z-IL that are obtained by ablating each one of the initial conditions.

After conducting one training step of each algorithm, we have computed the Euclidean distance between the weights obtained by one of the algorithms (ii)--(v), and the ones obtained by BP. The results of these experiments, reported in Table~\ref{tb:abl}, show that all the three conditions of Z-IL are necessary in order to achieve zero divergence with BP.
To provide full evidence of the validation of our theoretical results, we have conducted this experiment using ANNs, CNNs, RNNs, ResNets, and Transformer networks. Further details about the experiments are given in the section below.

\section{Reproducibility of the Experiments} 
In this section, we provide the details of all the experiments shown in the experimental sections.

\paragraph{MLPs:} 
To perform our experiments with fully connected networks (multilayer perceptrons), we have trained three architectures with different depth on FashionMNIST. Particularly, these networks have a hidden dimension of $128$ neurons, and $2,3$, and $4$ layers, respectively. We have used a  batch of $20$ training points, and a learning rate of $0.01$. The numbers reported for the experiments are the averages over the three architectures. 

\paragraph{CNNs:} For our experiments on CNNs, we have used AlexNet trained on both FashionMNIST and ImageNet. As above, we have used a  batch of $20$ training points, a learning rate of $0.01$, and reported the average of the experiments over the two datasets. 

\paragraph{RNNs:} We have trained a reinforcement learning agent on a single-layer many-to-one RNN, with $n = n^{out} = 128$, on eight different Atari games. Batch size and learning rate are $32$ and $0.001$, respectively. Again, the reported results are the average of all the experiments performed on this architecture.

\paragraph{ResNets:} We have used a $5$-layers fully connected network with $256$ hidden neurons per layer. The residual connections are defined at every layer. Particularly, we have defined it in a way that allows its computational graph to be levelled. 

\paragraph{Transformer:} We have used a single-layer transformed architecture, trained on randomly generated data.

\paragraph{Hardware:} All experiments are conducted on 2 Nvidia GeForce GTX 1080Ti  GPUs and 8 Intel Core i7 CPUs, with 32 GB RAM. Furthermore, to avoid rounding errors, we have initialized the weights in \emph{float32}, and then transformed them in \emph{float64}, and all later computations are in \emph{float64}.

\section{Convolutional Networks}

\emph{Convolutional neural networks} (\emph{CNNs}) are biologically inspired networks with a connectivity pattern (given by a set of \emph{kernels}) that resembles the structure of animals' visual cortex. Networks with this particular architecture are widely used in image recognition tasks. A CNN is formed by a sequence of convolutional layers, followed by a sequence of fully connected ones. For simplicity of notation, we now consider convolutional layers with one kernel, then, we show how to extend our results to the general case. We now recall the structure of CNNs and compare it against convolutional~PCNs.

\subsection{CNNs Trained with BP}

The learnable parameters of a convolutional layer are contained in different kernels. Each kernel $\bar \rho^{\scriptscriptstyle {l}}$ can be seen as a vector of dimension $m$,
which acts on the input vector $ f(\bar y^{\scriptscriptstyle {l}})$ using an operation ``$*$'', called \emph{convolution}, which is equivalent to a linear transformation with a sparse matrix~$w^{\scriptscriptstyle {l}}$, whose non-zero entries equal to the entries of the kernel $\bar \rho^{\scriptscriptstyle {l}}$. This particular matrix is called doubly-block circulant matrix \cite{Sedghi20}. For every entry $\rho^{\scriptscriptstyle {l}}_a$ of a kernel, we denote by $\mathcal C^{\scriptscriptstyle {l}}_a$ the set of indices $(i,j)$ such that~$w^{\scriptscriptstyle {l}}_{i,j} = \rho^{\scriptscriptstyle {l}}_a$.

Let $f(\bar y^{\scriptscriptstyle {l+1}})$ be the input of a convolutional layer with kernel $\bar{\rho}^{\scriptscriptstyle {l}}$. The output $\bar y^{\scriptscriptstyle {l}}$ can then be computed as in the fully connected case: it suffices to use Eq.~\eqref{eq:delta-recursive}, where~$w^{\scriptscriptstyle {l}}$ is the doubly-block circulant matrix with parameters in $\overline \rho^{\scriptscriptstyle {l}}$. During the learning phase, BP updates the parameters of $\bar \rho^{\scriptscriptstyle {l+1}}$ according to the following equation:
\begin{equation}
\Delta \rho^{\scriptscriptstyle {l+1}}_{a} 
= -\alpha\cdot {\partial E}/{\partial \rho^{\scriptscriptstyle {l+1}}_{a}}
= {\textstyle\sum}_{(i,j) \in \mathcal C^{\scriptscriptstyle {l+1}}_a} \Delta w^{\scriptscriptstyle {l+1}}_{i,j}.
\label{eq:cnn_weight}
\end{equation}
The value $\Delta w^{l}_{i,j}$ can be computed using Eq.~\eqref{eq:deltaz}.

\subsection{Predictive Coding CNNs Trained with IL}

\begin{figure*}[t]
    \centering
	\includegraphics[width=0.4\textwidth]{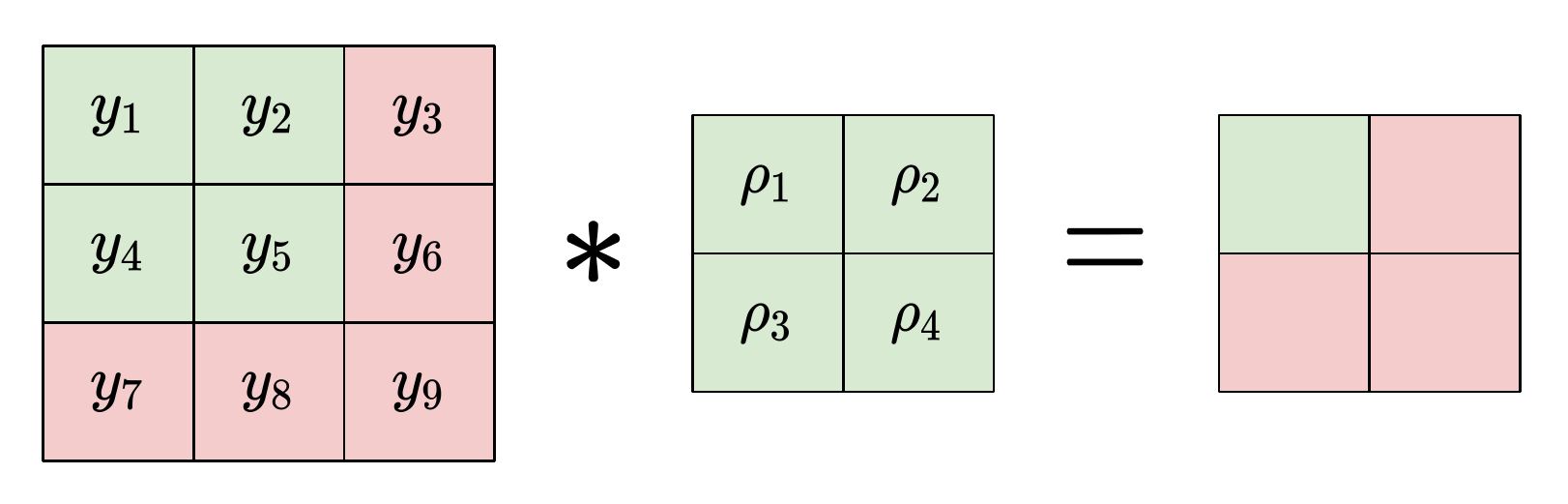}\vspace*{-1.5ex}
	\caption{An example of a \emph{convolution}.}\vspace*{-0.5ex}
	\label{fig:cnn}
\end{figure*}

Given a convolutional network, we call  $\overline{\lambda}^{\scriptscriptstyle {l}}$ the kernels of dimension $k$, and $\theta^{\scriptscriptstyle {l}}$ the related double-block circular matrix, which describes the convolution operation. Note that $\overline{\lambda}^{\scriptscriptstyle {l}}$ mirrors $\overline \rho^{\scriptscriptstyle {l}}$ defined in CNNs, while $\theta^l$ mirrors $w^l$. The dynamics of the feedforward pass is the same as the one described in the fully connected case. Hence,  the quantities $\varepsilon^{\scriptscriptstyle {l}}_{i,t}$ and $ \Delta x_{i,t}^{\scriptscriptstyle {l}}$ are computed as in Eqs.~\eqref{eq:mu-cg} and \eqref{eq:deltax-cg}. The update of the entries of the kernels, on the other hand, is the following:
\begin{equation}
    \Delta {\lambda}_a^{\scriptscriptstyle {l+1}} = -\alpha \cdot \frac{\partial F_t}{\partial {\lambda}_a^{\scriptscriptstyle {l+1}}} =  {\textstyle\sum}_{(i,j) \in \mathcal C^{\scriptscriptstyle {l+1}}_a} \Delta \theta_{i,j}^{\scriptscriptstyle {l+1}}, 
    \label{eq:cnn_il}
\end{equation}
where $\Delta \theta_{i,j}^{\scriptscriptstyle {l}}$ is computed according to Eq.~\eqref{eq:deltazeta}.

\subsection{Predictive Coding CNNs Trained with Z-IL}

Above, we have described the training and prediction phases on a single point $s = (\bar s^{\text{in}},\bar s^{\text{out}})$ under different architectures. The training phase of IL on a single point runs for $T$ iterations, during which the inference of Eq.~\eqref{eq:deltax-cg} is conducted, and $T$ is a hyperparameter that is usually set to different sufficiently large values  to get inference converged \cite{whittington2017approximation,millidge2020predictive}. So, the inference phase starts at $t\,{=}\,0$ and ends at $t\,{=}\,T$, which is also when the 
network 
parameters are updated  via 
Eqs.~\eqref{eq:deltaz} and~\eqref{eq:cnn_il}. 

To show that IL is able to do exact BP on both the fully connected and convolutional layers of a CNN, we add constraints on the weights update of IL. 

\smallskip 
\noindent\textbf{Z-IL:} Let $M$ be a PCN model with $l_{\text{max}}$ layers. The inference phase runs for $T = l_{\text{max}}$ iterations. Instead of updating all the weights simultaneously at $t=T$, the parameters of every layer $\theta^{l+1}$ are updated at $t=l$. Hence, the prediction phase of Z-IL is equivalent to the one of IL, while the learning phase updates the parameters according to the  equation
\begin{equation}
\!\Delta{\theta}^{\scriptscriptstyle {l+1,t}}_{i,j} 
= \begin{cases}
0  &  \!\!\mbox{if } t \neq l \\
\alpha\cdot \varepsilon^{\scriptscriptstyle {l}}_{i}
f ( x^{\scriptscriptstyle {l+1}}_{j,t} ) &  \!\!\mbox{if } t = l.
\end{cases}
\label{eq:pcn-dotx}
\end{equation}%
We now show that, under a specific choice of hyperparameters, Z-IL is equivalent to BP on CNNs. Particularly, we add the following two conditions: $\varepsilon^l_{i,0} \,{=}\, 0$ for $l\,{>}\,0$, and $\gamma  \,{=}\, 1$. 

 The first condition can be obtained by setting $\bar x^l_{0} \,{=}\, \bar \mu^l_{0}$ for every $l\,{>}\,0$ at the start of inference. Considering the vector $\bar \mu^l_0$ is computed from $\bar x^{l+1}_0$, this allows IL to start from a prediction stage  equivalent to the one of~BP. The second condition guarantees the propagation of the error during the inference phase to match the one of BP. Without it, 
 Z-IL would be equivalent to a variation of BP, where the weight updates of single layers would have different learning rates. 

The following theorem shows that Z-IL on convolutional PCNs is equivalent to BP on classical CNNs.

\begin{theorem}\label{thm:cnn}
    Let $M$ be a convolutional PCN trained with Z-IL with $\gamma \,{=}\, 1$ and $\varepsilon^l_{i,0} \,{=}\, 0$ for $l\,{>}\,0$, and let $M'$ be its corresponding CNN, initialized as $M$ and trained with BP.
   Then, given the same datapoint $s$ to both networks, we have 
   \begin{equation}
    \Delta \theta^{\scriptscriptstyle {l+1}}_{i,j} = \Delta w^{\scriptscriptstyle {l+1}}_{i,j}  \ \ \ \mbox{and} \ \ \ \Delta {\lambda}^{l+1}_i = \ \Delta \rho^{l+1}_i,
    \end{equation}
    for every $i,j,l \geq 0$.
\end{theorem}

\begin{proof}[Proof]
A convolutional network is formed by a sequence of convolutional layers followed by a sequence of fully connected ones. First, we prove the following:

\medskip 
\noindent\textbf{Claim 1:} At $t\,{=}\,l$, we have $\varepsilon^{\scriptscriptstyle {l}}_{i,t} = \delta^{\scriptscriptstyle {l}}_{i}$. 

\medskip 
\noindent This first partial result is proven by induction on the depth $l$ of the two networks, and does not change whether the layer considered is convolutional or fully connected. For PCNs, as $t=l$, it is also inducing on the inference moments. We begin by noting that, in Z-IL, $\varepsilon^{\scriptscriptstyle {l}}_{i,t} =  \varepsilon^{\scriptscriptstyle {l}}_{i,l}$.
\begin{itemize}
    \item \emph{Base Case, $l=0$}:

        The condition $\varepsilon^l_{i,0} \,{=}\, 0$ gives us $\mu^{\scriptscriptstyle {l}}_{i,0}=y^{\scriptscriptstyle {l}}_{i}$. Placing this result into Eq.~\eqref{eq:mu-cg}
        and Eq.~\eqref{eq:delta-recursive}, we get
        $\varepsilon^{\scriptscriptstyle {l}}_{i,l}=\delta^{\scriptscriptstyle {l}}_{i}$.
    \item \emph{Induction Step:}. For $l \in \lbrace 1, \ldots , l_{\text{max}}-1 \rbrace$, we have: 
        \begin{align*}
            & \varepsilon^{\scriptscriptstyle {l}}_{i,l}= f' ( \mu^{\scriptscriptstyle {l}}_{i,0} ) {\textstyle\sum}_{k=1}^{n^{\scriptscriptstyle {l-1}}} \varepsilon^{\scriptscriptstyle {l-1}}_{k,l-1} \theta^{\scriptscriptstyle {l}}_{k,i}   \text{ \ \ \ by Lemma~\ref{lem:pcn-varepsilon-iterative-app}} \\
            & \delta^{\scriptscriptstyle {l}}_{i} =f' ( y^{\scriptscriptstyle {l}}_{i} ) {\textstyle\sum}_{k=1}^{n^{\scriptscriptstyle {l-1}}} \delta^{\scriptscriptstyle {l-1}}_{k} w^{\scriptscriptstyle {l}}_{k,i}  \text{\ \ \ \ \ \ \ \ \ \ by Eq.~\eqref{eq:delta-recursive}.}
            \label{eq:error}
        \end{align*}
    Furthermore, note that $w_{i,j}^{\scriptscriptstyle {l}}=\theta_{i,j}^{\scriptscriptstyle {l}}$, because of the same initialization of the network, and $\mu^{\scriptscriptstyle {l}}_{i,0} = y^{\scriptscriptstyle {l}}_{i}$, because of $\varepsilon^l_{i,0} \,{=}\, 0$ for $l\,{>}\,0$. Plugging these two equalities into the error equations above gives \begin{equation}
        \varepsilon^{\scriptscriptstyle {l}}_{i,l}=\delta^{\scriptscriptstyle {l}}_{i}, \ \ \text{if} \ \ \varepsilon^{\scriptscriptstyle {l-1}}_{k,l-1}=\delta^{\scriptscriptstyle {l-1}}_{k}.
    \end{equation}
    This concludes the induction step and proves the claim.
\end{itemize}

We now have to show the equivalence of the weights updates. We  start our study from fully connected layers.

\medskip 
\noindent\textbf{Claim 2:} We have  $\Delta \theta^{\scriptscriptstyle {l+1}}_{i,j} = \Delta w^{\scriptscriptstyle {l+1}}_{i,j}$ for every $i,j,l \geq 0$. 

\medskip 
\noindent Eqs.~\eqref{eq:deltaz} and \eqref{eq:deltaz} state the following:

\begin{align*}
    & \Delta \theta^{\scriptscriptstyle {l+1}}_{i,j}
    = \alpha\cdot \varepsilon^{\scriptscriptstyle {l}}_{i,l} f ( x^{\scriptscriptstyle {l+1}}_{j,l} ), \\
    & \Delta w^{\scriptscriptstyle {l+1}}_{i,j} 
    =\alpha\cdot \delta^{\scriptscriptstyle {l}}_{i}
    f ( y^{\scriptscriptstyle {l+1}}_{j} ).
\end{align*}

Claim 1 gives $\varepsilon^{\scriptscriptstyle {l}}_{i,l}=\delta^{\scriptscriptstyle {l}}_{i}$.
We now have to show that ${\textstyle {\textstyle f ( x^{\scriptscriptstyle {l+1}}_{j,l} )=f ( y^{\scriptscriptstyle {l+1}}_{j} )}}$. 
The equivalence of the initial state between IL and BP gives ${x^{\scriptscriptstyle {l+1}}_{j,0} = \mu^{\scriptscriptstyle {l+1}}_{j,0} = \textstyle y^{\scriptscriptstyle {l+1}}_{j}}$.
Then, Lemma~\ref{lem:pcn-propagate-zero-app} shows that ${\textstyle x^{\scriptscriptstyle {l+1}}_{j,l}} = x^{\scriptscriptstyle {l+1}}_{j,0}$.
So,  ${\textstyle f ( x^{\scriptscriptstyle {l+1}}_{j,l} )=f ( y^{\scriptscriptstyle {l+1}}_{j} )}$.

\medskip 
\noindent\textbf{Claim 3:} We have  $\Delta \lambda^{\scriptscriptstyle {l+1}}_{a} = \Delta \rho^{\scriptscriptstyle {l+1}}_{a}$ for every $a,l \geq 0$. 

\medskip 
\noindent The law that regulates the updates of the kernels is given by the following equations:

\begin{align}\label{eq:cnn-proof}
    &\Delta {\lambda}_a^{l+1} = -\alpha \cdot \frac{\partial F_t}{\partial {\lambda}_a^{l}} =  {\textstyle\sum}_{(i,j) \in \mathcal C^{l+1}_a} \Delta \theta_{i,j}^{l} \\
    & \Delta \rho^{\scriptscriptstyle {l+1}}_{a} 
    = -\alpha\cdot {\partial E}/{\partial \rho^{\scriptscriptstyle {l+1}}_{a}}
    = {\textstyle\sum}_{(i,j) \in \mathcal C^{l+1}_a} \Delta w^{l+1}_{i,j}.
    \label{eq:cnn-proof-il}
\end{align}
These equations are equal if $\Delta \theta^{\scriptscriptstyle {l}}_{i,j} = \Delta w^{\scriptscriptstyle {l}}_{i,j}$ for every $i,j,l >0$, which is the result shown in Claim 2. Thus, the weight update at every iteration of Z-IL is equivalent to the one of BP for both convolutional and fully connected layers. 
\end{proof}

\begin{lemma}
    \label{lem:pcn-propagate-zero-app}
    Let $M$ be a convolutional PCN trained with Z-IL with $\gamma \,{=}\, 1$ and $\varepsilon^l_{i,0} \,{=}\, 0$ for $l\,{>}\,0$. Then, a variable $\bar x^l_{t}$ can only diverge from its corresponding initial state at time $t=l$. Formally,
    \begin{align*}
        & \overline{x}^{\scriptscriptstyle {l}}_{t<l}=\overline{x}^{\scriptscriptstyle {l}}_{0}, \overline{\varepsilon}^{\scriptscriptstyle {l}}_{t<l}=\overline{\varepsilon}^{\scriptscriptstyle {l}}_{0}=0, \overline{\mu}^{\scriptscriptstyle {l-1}}_{t<l}=\overline{\mu}^{\scriptscriptstyle {l-1}}_{0},  \text{ i.e.,} \\
        & \Delta{\overline{x}}^{\scriptscriptstyle {l}}_{t<l-1}=\overline{0}, \Delta{\overline{\varepsilon}}^{\scriptscriptstyle {l}}_{t<l-1}=\overline{0}, \Delta{\overline{\mu}}^{\scriptscriptstyle {l-1}}_{t<l-1}=\overline{0}
  \end{align*}
for  $l \in \lbrace 1, \ldots , l_{\text{max}}-1\rbrace$.
\end{lemma}

\begin{proof}[Proof]
    Starting from the inference moment $t=0$, $\overline{x}^{\scriptscriptstyle {0}}_{0}$ is dragged away from $\overline{\mu}^{\scriptscriptstyle {0}}_{0}$ and fixed to $\overline{s}^{\text{out}}$, i.e., $\overline{\varepsilon}^{\scriptscriptstyle {0}}_{0}$ turns into nonzero from zero. 
    Since $\overline{x}$ in each layer is updated only on the basis of $\overline{\varepsilon}$ in the same and previous adjacent layer, as indicated by Eq. \eqref{eq:pcn-dotx}, also considering that $\varepsilon^l_{i,0} \,{=}\, 0$, for all layers but the output layer, it will take $l$ time steps to modify $\overline{x}^{\scriptscriptstyle {l}}_{t}$ at layer $l$ from the initial state. 
    Hence, $\overline{x}^{\scriptscriptstyle {l}}_{t}$ will remain in that initial state $\overline{x}^{\scriptscriptstyle {l}}_{0}$ for all $t<l$, i.e., $\overline{x}^{\scriptscriptstyle {l}}_{t<l}=\overline{x}^{\scriptscriptstyle {l}}_{0}$.
    Furthermore, any change in $\overline{x}^{\scriptscriptstyle {l}}_{t}$ causes a change in $\overline{\varepsilon}^{\scriptscriptstyle {l}}_{t}$ and $\overline{\mu}^{\scriptscriptstyle {l-1}}_{t}$ instantly via Eq. \eqref{eq:mu-cg} (otherwise $\overline{\varepsilon}^{\scriptscriptstyle {l}}_{t}$ and $\overline{\mu}^{\scriptscriptstyle {l-1}}_{t}$ remain in their corresponding initial states). 
    Thus, we know $\overline{\varepsilon}^{\scriptscriptstyle {l}}_{t<l}=\overline{\varepsilon}^{\scriptscriptstyle {l}}_{0}$ and $\overline{\mu}^{\scriptscriptstyle {l-1}}_{t<l}=\overline{\mu}^{\scriptscriptstyle {l-1}}_{0}$.
    Also, according to Eq.~\eqref{eq:pcn-dotx}, $\overline{\varepsilon}^{\scriptscriptstyle {l}}_{t<l}=\overline{\varepsilon}^{\scriptscriptstyle {l}}_{0}=0$.
    Equivalently, we have $\Delta{\overline{x}}^{\scriptscriptstyle {l}}_{t<l-1}=\overline{0}$, $\Delta{\overline{\varepsilon}}^{\scriptscriptstyle {l}}_{t<l-1}=\overline{0}$, and $\Delta{\overline{\mu}}^{\scriptscriptstyle {l-1}}_{t<l-1}=\overline{0}$.
\end{proof}

\begin{lemma}
    \label{lem:pcn-varepsilon-iterative-app}
    Let $M$ be a convolutional PCN trained with Z-IL with $\gamma \,{=}\, 1$ and $\varepsilon^l_{i,0} \,{=}\, 0$ for $l\,{>}\,0$. Then, the prediction error $\varepsilon^{\scriptscriptstyle {l}}_{i,t}$ at $t=l$ (i.e., $\varepsilon^{\scriptscriptstyle {l}}_{i,l}$)  can be derived from itself at previous inference moments in the previous layer.
    Formally:
     \begin{align}
        \varepsilon^{\scriptscriptstyle {l}}_{i,l}= f' ( \mu^{\scriptscriptstyle {l}}_{i,0} ) {\textstyle\sum}_{k=1}^{n^{\scriptscriptstyle {l-1}}} \varepsilon^{\scriptscriptstyle {l-1}}_{k,l-1} \theta^{\scriptscriptstyle {l}}_{k,i} ,  \label{eq:pcn-varepsilon-iterative-app}
    \end{align}
for $l \in \lbrace 1, \ldots , l_{\text{max}}-1 \rbrace\,.$
\end{lemma}

\begin{proof}[Proof]
    We first write a dynamic version of 
    $\varepsilon^{\scriptscriptstyle {l}}_{i,t} = x^{\scriptscriptstyle {l}}_{i,t} - \mu^{\scriptscriptstyle {l}}_{i,t}$:
    \begin{equation}
        \varepsilon^{\scriptscriptstyle {l}}_{i,t} = \varepsilon^{\scriptscriptstyle {l}}_{i,t-1} + {(\Delta{x}^{\scriptscriptstyle {l}}_{i,t-1} - \Delta{\mu}^{\scriptscriptstyle {l}}_{i,t-1})\,,}
    \end{equation}
    where $\Delta{\mu}^{\scriptscriptstyle {l}}_{i,t-1}=\mu^{\scriptscriptstyle {l}}_{i,t}-\mu^{\scriptscriptstyle {l}}_{i,t-1}$.
    Then, we expand $\varepsilon^{\scriptscriptstyle {l}}_{i,l}$ with the above equation and simplify it with Lemma~\ref{lem:pcn-propagate-zero-app}, i.e., $\varepsilon^{\scriptscriptstyle {l}}_{i,t<l}=0$ and $\Delta{\mu}^{\scriptscriptstyle {l-1}}_{i,t<l-1}=0$:
    \begin{align}
        \varepsilon^{\scriptscriptstyle {l}}_{i,l} 
        = \varepsilon^{\scriptscriptstyle {l}}_{i,l-1} + {(\Delta{x}^{\scriptscriptstyle {l}}_{i,l-1} - \Delta{\mu}^{\scriptscriptstyle {l}}_{i,l-1})}
        ={\Delta{x}^{\scriptscriptstyle {l}}_{i,l-1}},. \label{eq:varepsilon-dotx}
    \end{align}
    for $l \in \lbrace 1, \ldots , l_{\text{max}}-1 \rbrace$. We further investigate $\Delta{x}^{\scriptscriptstyle {l}}_{i,l-1}$ expanded with the inference dynamic Eq.~\eqref{eq:pcn-dotx} and simplify it with Lemma~\ref{lem:pcn-propagate-zero-app}, i.e., $\varepsilon^{\scriptscriptstyle {l}}_{i,t<l}=0$,
    \begin{align}
            \Delta{x}^{\scriptscriptstyle {l}}_{i,l-1}
            = & \gamma ( -\varepsilon^{\scriptscriptstyle {l}}_{i,l-1} + f' ( x^{\scriptscriptstyle {l}}_{i,l-1} ) ){\textstyle\sum}_{k=1}^{n^{\scriptscriptstyle {l-1}}} \varepsilon^{\scriptscriptstyle {l-1}}_{k,l-1} \theta^{\scriptscriptstyle {l}}_{k,i} \\
            = &\gamma f' ( x^{\scriptscriptstyle {l}}_{i,l-1} ){\textstyle\sum}_{k=1}^{n^{\scriptscriptstyle {l-1}}} \varepsilon^{\scriptscriptstyle {l-1}}_{k,l-1} \theta^{\scriptscriptstyle {l}}_{k,i}, 
        \label{eq:dotx-itertive}
    \end{align}
    for $l \in \lbrace 1, \ldots , l_{\text{max}}-1 \rbrace$. Putting Eq.~\eqref{eq:dotx-itertive} into Eq.~\eqref{eq:varepsilon-dotx}, we obtain:
    \begin{align}
        \varepsilon^{\scriptscriptstyle {l}}_{i,l} 
        =\gamma f' ( x^{\scriptscriptstyle {l}}_{i,l-1} ){\textstyle\sum}_{k=1}^{n^{\scriptscriptstyle {l-1}}} \varepsilon^{\scriptscriptstyle {l-1}}_{k,l-1} \theta^{\scriptscriptstyle {l}}_{k,i},
    \end{align}
    for $l \in \lbrace 1, \ldots , l_{\text{max}}-1 \rbrace$. With Lemma~\ref{lem:pcn-propagate-zero-app}, $x^{\scriptscriptstyle {l}}_{i,l-1}$ can be replaced with $x^{\scriptscriptstyle {l}}_{i,0}$.
    With $\varepsilon^l_{i,0} \,{=}\, 0$ for $l\,{>}\,0$, we can further replace $x^{\scriptscriptstyle {l}}_{i,0}$ with $\mu^{\scriptscriptstyle {l}}_{i,0}$.
    Thus, the above equation becomes:
    \begin{align}
        \varepsilon^{\scriptscriptstyle {l}}_{i,l} 
        ={\gamma} f' ( \mu^{\scriptscriptstyle {l}}_{i,0} ){\textstyle\sum}_{k=1}^{n^{\scriptscriptstyle {l-1}}} \varepsilon^{\scriptscriptstyle {l-1}}_{k,l-1} \theta^{\scriptscriptstyle {l}}_{k,i}, \label{eq:pcn-varepsilon-iterative-app-with-gamma}
    \end{align}
    for $l \in \lbrace 1, \ldots , l_{\text{max}}-1 \rbrace$. Then, put $\gamma=1$, into the above equation.
\end{proof}

\section{Extension to the Case of Multiple Kernels per Layer}

In the theorem proved in the previous section, 
we have only considered CNNs with one kernel per layer. While networks of this kind are theoretically interesting, in practice a convolutional layer is made of multiple kernels. We now show that the result of Theorem~\ref{thm:cnn} still holds if we consider networks of this kind. Let $M_l$ be the number of kernels present in layer $l$. In Theorem~\ref{thm:cnn}, we have considered the case $M_l = 1$ for every convolutional layer. Consider now the following three cases: 
\begin{itemize}
    \item \textbf{Case 1: $M_l > 1, M_{l-1} = 1$}. We have a network with a convolutional layer at position $l$ with $M_l$ different kernels $\{ \bar \rho^{l,1}, \dots, \bar \rho^{l,M_l} \}$ of the same size $k$. The result of the convolution between the input $f(\bar y^{l})$ and a single kernel $\bar\rho^{l,m}$ is called \emph{channel}. The final output $\bar y^{l-1}$ of a convolutional layer is obtained by concatenating all the channels into a single vector. 
    The operation generated by convolutions and concatenation just described, can be written as a linear map $w^l \cdot f(\bar y^{l})$, where the matrix $w^l$ is formed by $M_l$ doubly-block circulant matrices stocked vertically, each of which has entries equal to the ones of a kernel $\bar \rho^{l,m}$. For each entry $\rho^{\scriptscriptstyle {l,m}}_a$ of each kernel in layer $l$, we denote by $\mathcal C^{\scriptscriptstyle {l}}_{m,a}$ the set of indices $(i,j)$ such that $w^{\scriptscriptstyle {l}}_{i,j} = \rho^{\scriptscriptstyle {l,m}}_a$. The equation describing the changes of parameters in the kernels is then the following:
    \begin{equation}
        \Delta \rho^{\scriptscriptstyle {l,m}}_{a} 
        = -\alpha\cdot {\partial E}/{\partial \rho^{\scriptscriptstyle {l,m}}_{a}}
        = {\textstyle\sum}_{(i,j) \in \mathcal C^{\scriptscriptstyle {l}}_{m,a}} \Delta w^{\scriptscriptstyle {l}}_{i,j}.
    \end{equation}
    
    \item \textbf{Case 2: $M_l = 1,M_{l-1} > 1$}. We now analyze what happens in a layer with only one kernel, when the input $f(\bar y^{l-1})$ comes from a layer with multiple kernels. This case differs from Case 1, because the input represents a concatenation of $M_{l-1}$ different channels. In fact, the kernel $\bar \rho^{l}$ gets convoluted with every channel independently. The resulting vectors of these convolutions are then summed together, obtaining $\bar y^{l}$. The operation generated by convolutions and summations just described, can be written as a linear map $w^l \cdot f(\bar y^{l})$. In this case, the matrix $w^l$ is formed by $M_{l-1}$ doubly-block circulant matrices stocked horizontally, each of which has entries equal to the ones of the kernel $\bar \rho^{l}$. For every entry $\rho^{\scriptscriptstyle {l}}_a$, we denote by $\mathcal C^{\scriptscriptstyle {l}}_{a}$ the set of indices $(i,j)$ such that $w^{\scriptscriptstyle {l}}_{i,j} = \rho^{\scriptscriptstyle {l}}_a$. The equation that describes the changes of parameters in the kernels is then the following:
    \begin{equation}
        \Delta \rho^{\scriptscriptstyle {l}}_{a} 
        = -\alpha\cdot {\partial E}/{\partial \rho^{\scriptscriptstyle {l}}_{a}}
        = {\textstyle\sum}_{(i,j) \in \mathcal C^{\scriptscriptstyle {l}}_{a}} \Delta w^{\scriptscriptstyle {l}}_{i,j}.
    \end{equation}
        
    \item \textbf{Case 3 (General Case): $M_l,M_{l-1} > 1$}. We now move to the most general case: a convolutional layer at position $l$ with $M_l$ different kernels $\{ \bar \rho^{l,1}, \dots, \bar \rho^{l,M_l} \}$, whose input $f(\bar y^l)$ is a vector formed by $M_{l-1}$ channels. In this case, every kernel does a convolution with every channel. The output $\bar y^{l+1}$ is obtained as follows: the results obtained using the same kernel on different channels are summed together, and concatenated with the results obtained using the other kernels. Again, this operation can be written as a linear map $w^l \cdot f(\bar y^{l})$. By merging the results obtained from Case 1 and Case 2, we have that the matrix $w^l$ is a grid of $M_l \times M_{l+1}$ doubly-block circulant submatrices. For every entry $\rho^{\scriptscriptstyle {l,m}}_a$ of every kernel in layer $l$, we denote by $\mathcal C^{\scriptscriptstyle {l}}_{m,a}$ the set of indices $(i,j)$ such that $w^{\scriptscriptstyle {l}}_{i,j} = \rho^{\scriptscriptstyle {l,m}}_a$. The equation  describing the changes of parameters in the kernels is then the following:
    \begin{equation}\label{eq:gen-cnn}
        \Delta \rho^{\scriptscriptstyle {l,m}}_{a} 
        = -\alpha\cdot {\partial E}/{\partial \rho^{\scriptscriptstyle {l,m}}_{a}}
        = {\textstyle\sum}_{(i,j) \in \mathcal C^{\scriptscriptstyle {l}}_{m,a}} \Delta w^{\scriptscriptstyle {l}}_{i,j}.
    \end{equation}
    
\end{itemize}

To integrate this general case in the proof of Theorem~\ref{thm:cnn}, it suffices to consider Eq.~\eqref{eq:gen-cnn}, and its equivalent formulation in the language of a convolutional PCN,
\begin{equation}
\Delta \rho^{\scriptscriptstyle {l+1}}_{a} 
    = -\alpha\cdot {\partial E}/{\partial \rho^{\scriptscriptstyle {l+1}}_{a}}
    = {\textstyle\sum}_{(i,j) \in \mathcal C^{l+1}_a} \Delta w^{l+1}_{i,j}
\end{equation}
 instead of Eqs.~\eqref{eq:cnn-proof} and \eqref{eq:cnn-proof-il}. Note that both equations are fully determined once we have computed $\Delta w^{\scriptscriptstyle {l}}_{i,j}$ and $\Delta \theta^{\scriptscriptstyle {l}}_{i,j}$ for every $i,j>0$. Hence, the result follows directly by doing the same computations.

\section{Recurrent Neural Networks (RNNs)}

While CNNs achieve impressive results in computer vision tasks, their performance drops when handling data with sequential structure, such as natural language sentences. An example is sentiment analysis: given a sentence $S^{\text{in}}$ with words $(\overline s^{\text{in}}_1, \dots, \overline s^{\text{in}}_N)$, predict whether this sentence is positive, negative, or neutral. To perform classification and regression tasks on this kind of data, the last decades have seen the raise of recurrent neural networks (RNNs). Networks that deal with a sequential input and a non-sequential output are called \emph{many-to-one} RNNs. An example of such an architecture is shown in Fig.~\ref{fig:rnn}.
In this section, we show that the proposed Z-IL, along with our conclusions, can be extended to RNNs as well.
We first recall RNNs trained with BP, and then show how to define a recurrent PCN trained with IL. We conclude  by showing that the proposed Z-IL can also be carried over and scaled to RNNs, and that our equivalence conclusions still hold.

\subsection{RNNs Trained with BP}

An RNN for classification and regression tasks has three different weight matrices $w^x, w^h$, and $ w^y$, $N$ hidden layers of dimension $n$, and an output layer {of dimension} $n^{\text{out}}$. When it does not lead to confusion, we will alternate the notation between $k=\text{out}$ and $k=N+1$. This guarantees a lighter notation in the formulas. A sequential input $S^{\text{in}} = \{ \overline s^{\text{in}}_{1}, \dots , \bar s^{\text{in}}_{N} \}$ is a sequence of $N$ vectors of dimension $n^{\text{in}}$. The first hidden layer is computed using the first vector of the sequential input, while  the output layer is computed by multiplying the last hidden layer by the matrix $w^y$, i.e.,  $\bar y^{\text{out}} = w^y \cdot f(\bar y^N) $. The structure of the RNN with the used notation is summarized in Fig.~\ref{fig:rnn}. By assuming $\overline y^0 = \overline 0$, the local computations of the network can be written~as follows:
\begin{equation}
    \begin{split}
    & y^k_i = {\textstyle\sum}_{j=1}^{n} w^h_{i,j} f ( y^{\scriptscriptstyle {k-1}}_{j} ) + {\textstyle\sum}_{j=1}^{n^{\text{in}}} w^x_{i,j} s^{\text{in}}_{k,j}, \\
    & y^{\text{out}}_i = {\textstyle\sum}_{j=1}^{n} w^y_{i,j} f ( y^{\scriptscriptstyle N}_{j} ).
    \end{split}
    \label{eq:rnn-forward}
\end{equation}

\smallskip 
\noindent\textbf{Prediction: }%
Given a sequential value
$S^{\text{in}}$ as input, every $y^{\scriptscriptstyle {k}}_{i}$ in the RNN is computed via Eq.~\eqref{eq:rnn-forward}.

\smallskip 
\noindent\textbf{Learning: }%
Given a sequential value
$S^{\text{in}}$ as input, the output $\bar y^{\text{out}}$ is then compared with the label $\overline s^{\text{out}}$ using MSE Loss.
We now show how BP updates the weights of the three weight matrices. Note that $w^y$ is a fully connected layer that connects the last hidden layer to the output layer. We have already computed this specific weight update in Eq.~\eqref{eq:deltaz}:
\begin{equation}
    \Delta w^y_{i,j} = \alpha \cdot \delta^{\text{out}}_i f(y^N_j) \mbox{ \ \ with \ \ } \delta^{\text{out}}_i =  s^{\text{out}}_{i} - y^{\text{out}}_{i}.
\end{equation}
The gradients of $E$ relative to the single entries of 
$w^x$ and $w^y$ are the sum of the gradients at each recurrent layer $k$. Thus,
\begin{equation}
    \begin{split}
        & \Delta w^x_{i,j} = \alpha \cdot {\textstyle\sum}_{k=1}^N \delta^{k}_i s^{\text{in}}_{k,j} \\
        & \Delta w^h_{i,j} = \alpha \cdot {\textstyle\sum}_{k=1}^N \delta^{k}_i f(y^{k-1}_j).
    \end{split}
    \label{eq:weight_rnn_bp}
\end{equation}
The error term $\delta^{\scriptscriptstyle {k}}_{i} \,{ =}\, {\partial E}/{\partial {y}^{\scriptscriptstyle {k}}_{i}}$ is defined as in~Eq.~\eqref{eq:delta-recursive}:

\begin{equation}
    \delta^{\scriptscriptstyle {k}}_{i} = f' ( y^{\scriptscriptstyle {k}}_{i} ) {\textstyle\sum}_{j=1}^{n} \delta^{\scriptscriptstyle {k+1}}_{j} w^{\scriptscriptstyle {h}}_{j,i}.
    \label{eq:error-rnn}
\end{equation}
\subsection{Predictive Coding RNNs Trained with IL}

\begin{figure*}
    \centering
	\includegraphics[width=0.3\textwidth]{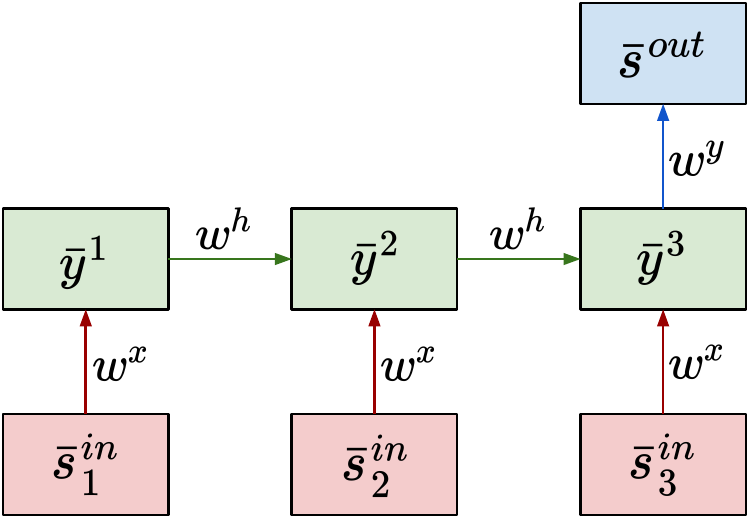}\vspace*{-1ex}
	\caption{An example of a \emph{many-to-one} RNN.}\vspace*{-1ex}
	\label{fig:rnn}
\end{figure*}

We show how to define a recurrent PCN trained with IL.
Recurrent PCNs have the same layer structure as the network introduced in the previous section. Hence, by assuming $\overline x^0 = \overline 0$, the forward pass is given by as follows: 
\begin{equation}
    \begin{split}
    & \mu^k_{i,t} = {\textstyle\sum}_{j=1}^{n} \theta^h_{i,j} f ( x^{\scriptscriptstyle {k-1}}_{j,t} ) + {\textstyle\sum}_{j=1}^{n} \theta^x_{i,j} s^{\text{in}}_{k,j}, \\
    & \mu^{\text{out}}_{i,t} = {\textstyle\sum}_{j=1}^{n} \theta^y_{i,j} f ( x^{\scriptscriptstyle N}_{j,t} ).
    \end{split}
    \label{eq:rnn-forward-il}
\end{equation}
Here, $\theta^x$, $\theta^h$, and $\theta^y$ are the weight matrices paralleling $w^x$, $w^h$, and $w^y$, respectively. The $\mu^{\scriptscriptstyle {k}}_i$ and $x^{\scriptscriptstyle {k}}_i$ are defined as in the preliminaries. Again, error nodes computes the error between them $\varepsilon^{\scriptscriptstyle {k}}_{i,t} = x^{\scriptscriptstyle {k}}_{i,t} - \mu^{\scriptscriptstyle {k}}_{i,t}$. During the inference phase, the value nodes $x^{\scriptscriptstyle {k}}_{i,t}$ are updated to minimize the energy function $F_t$. During the learning phase, this update is done via:
\begin{equation}
\!\Delta{x}^{\scriptscriptstyle {k}}_{i,t} = \begin{cases}
\gamma\cdot ( -\varepsilon^{\scriptscriptstyle {k}}_{i,t} + f' ( x^{\scriptscriptstyle {k}}_{i,t} ) {\textstyle\sum}_{j=1}^{n} \varepsilon^{\scriptscriptstyle {k-1}}_{j,t} \theta^{\scriptscriptstyle {h}}_{j,i} ) & \!\!\mbox{if } k \geq 1\\
0 & \!\!\mbox{if } k = \text{out}.
\end{cases}
\label{eq:rnn-pcn-dotx}
\end{equation}%

\smallskip 
\noindent\textbf{Prediction: }%
Given a sequential value
$S^{\text{in}}$ as input, every $\mu^{\scriptscriptstyle {k}}_{i}$ in the RNN is computed as the prediction via Eq.~\eqref{eq:rnn-forward-il}. Again, all error nodes converge to zero when $t\rightarrow \infty$, thus, $x^{\scriptscriptstyle {k}}_{i} = \mu^{\scriptscriptstyle {k}}_{i}$. 

\smallskip 
\noindent\textbf{Learning: }%
Given a sequential value
$S^{\text{in}}$ as input, the error in the output layer is set to  $\varepsilon^{\text{out}}_{i,0} = s^{\text{out}}_{i} - \mu^{\text{out}}_{i,0}$. From here, the inference phase spreads the error among all the neurons of the network. Once this process has converged to an equilibrium, the parameters of the network get updated in order to minimize the total energy function. This causes the following weight updates:
\begin{equation}
    \begin{split}
        & \Delta \theta^x_{i,j} = \alpha \cdot {\textstyle\sum}_{k=1}^N \varepsilon^{k}_{i,t} s^{\text{in}}_{k,j} \\
        & \Delta \theta^h_{i,j} = \alpha \cdot {\textstyle\sum}_{k=1}^N \varepsilon^{k}_{i,t} f(x^{k-1}_j) \\
        & \Delta \theta^y_{i,j} = \alpha \cdot \varepsilon^{\text{out}}_{i,t} f(x^N_j).
    \end{split}
    \label{eq:weight_rnn_il}
\end{equation}

\subsection{Predictive Coding RNNs Trained with Z-IL}

We now show that Z-IL can also be carried over and scaled to RNNs, and that the equivalence of Theorem~\ref{thm:1} also holds for the considered RNNs. This equivalence can be extended to deeper networks, as it suffices to stack multiple layers (fully connected or convolutional) on top of the RNN's output layer.

\begin{theorem}
    \label{ther:final-equal-rnn-supp} 
    Let $M$ be a recurrent PCN trained with Z-IL with $\gamma \,{=}\, 1$ and $\varepsilon^k_{i,0} \,{=}\, 0$ for $k \,{>}\,0$, and let $M'$ be its corresponding RNN, initialized as $M$ and trained with BP.
   Then, given the same sequential input $S \,{=}\, \{\bar s_1, \dots, \bar s_N\}$ to both, 
\begin{equation}
    \begin{split}
        & \Delta \theta^x_{i,j} = \Delta w^x_{i,j} \\
        & \Delta \theta^h_{i,j} = \Delta w^h_{i,j} \\
        & \Delta \theta^y_{i,j} = \Delta w^y_{i,j},
    \end{split}
    \label{eq:weight_rnn_theo}
\end{equation}
    for every $i,j>0$.
\end{theorem}

\begin{proof}

The network $M$ has depth $2$; hence, we set $T=2$. We now prove the following three equivalences: $(1)$ $\Delta \theta^y = \Delta w^y$, $(2)$ $\Delta \theta^h = \Delta w^h$, and $(3)$ $\Delta \theta^x = \Delta w^x$.

The proof of $(1)$ is straightforward, since both the output layers $\theta^y$ and $w^y$ are fully connected. Particularly, we have already shown the equivalence for this kind of layers in Theorem~\ref{thm:cnn}. Before proving $(2)$ and $(3)$, we show an intermediate result needed in both cases.

\medskip 
\noindent \textbf{Claim:} Given a sequential input $S^{in}$ of length $N$, at $t=1$ we have $\varepsilon^k_{i,1} = \delta^k_i$ for every $k \leq N$. 

\medskip 
\noindent This part of the proof is done by induction on $N$. 
\begin{itemize}
    \item Base Case: $N = 1$. Given a sequential input of length $1$, we have a fully connected network of depth $2$ with $w^1 = w^y$ (resp. $\theta^1 = \theta^y$) and $w^2 = w^x$ (resp. $\theta^2 = \theta^x$). We have already proved this result in Theorem~\ref{thm:cnn}.
    
    \item Induction Step. Let us assume that, given a sequential input $S^{in}$ of length $N$, the claim $\varepsilon^k_{i,1} = \delta^k_i$ holds for every $k \in \{1, \dots, N\}$. Let us now assume we have a sequential input of length $N+1$. Note that the errors  $\varepsilon^k_{i,1}$ and  $\delta^k_i$ are computed backwards starting from $k=N+1$. Hence, the quantities $\varepsilon^k_{i,1}$ and  $\delta^k_i$ for $k \in \{2, \dots, N+1\}$ are computed as they were the errors of a sequential input of length $N$. It follows by the induction argument that $\varepsilon^k_{i,1} = \delta^k_i$ for every $k \in \{2, \dots, N+1\}$. To conclude the proof, we have to show that $\varepsilon^1_{i,1} = \delta^1_i$.  For $k=1$, we have:
            \begin{align*}
            & \varepsilon^{\scriptscriptstyle {1}}_{i,l}= f' ( \mu^{\scriptscriptstyle {1}}_{i,0} ) {\textstyle\sum}_{j=1}^{n} \varepsilon^{\scriptscriptstyle {2}}_{j,t} \theta^{\scriptscriptstyle {h}}_{j,i}   \text{ \ \ \ \ \ \ \ by Lemma~\ref{lem:pcn-varepsilon-iterative-app-rnn}} \\
            & \delta^{\scriptscriptstyle {1}}_{i} = f' ( y^{\scriptscriptstyle {1}}_{i} ) {\textstyle\sum}_{j=1}^{n} \delta^{\scriptscriptstyle {2}}_{j} w^{\scriptscriptstyle {h}}_{j,i}.  \text{\ \ \ \ \ \ \ \ \ \ \  by Eq.~\eqref{eq:error-rnn}.}
            \label{eq:error-rnn-proof}
        \end{align*}
    
    Note that $w^h_{i,j}=\theta^h_{i,j}$, because of the same initialization of the network. Furthermore, $\mu^{\scriptscriptstyle {k}}_{i,0} = y^{\scriptscriptstyle {k}}_{i}$ for every $k$ because of $\varepsilon^k_{i,0} \,{=}\, 0$. Plugging these two equalities into the error equations above gives $\varepsilon^1_{i,1} = \delta^1_i$.
    This concludes the induction step and proves the claim.
\end{itemize}

\smallskip\noindent\emph{(2)  $\Delta \theta^h = \Delta w^h$.} Recall that Eqs.~\eqref{eq:weight_rnn_il} and \eqref{eq:weight_rnn_bp} state that
\begin{align*}
    & \Delta \theta^h_{i,j} = \alpha \cdot {\textstyle\sum}_{k=1}^N \varepsilon^{k}_{i,t} f(x^{k-1}_{j,1}) \\
    & \Delta w^h_{i,j} = \alpha \cdot {\textstyle\sum}_{k=1}^N \delta^{k}_i f(y^{k-1}_j).
\end{align*}
The claim shown above gives $\varepsilon^{k}_{i,1}  \,{=}\, \delta^{k}_i$.  We thus have to show that $x^{k}_{j,1}  \,{=}\,y^{k}_j$. The condition $\varepsilon^k_{j,0}  \,{=}\, 0$ gives $x^{k}_{j,0} \,{=}\, \mu^{k}_{j,0} \,{=}\, y^{k}_{j}$. Moreover, by Lemma~\ref{lem:pcn-varepsilon-iterative-app}, $x^{k}_{j,1}  \,{=}\, x^{k}_{j,0}$. So, $x^{k}_{j,1} \,{=}\, y^{k}_j$. 

\smallskip\noindent\emph{(3)  $\Delta \theta^x = \Delta w^x$.} Recall that Eqs.~\eqref{eq:weight_rnn_il} and \eqref{eq:weight_rnn_bp} state that
\begin{align*}
    & \Delta \theta^x_{i,j} = \alpha \cdot {\textstyle\sum}_{k=1}^N \varepsilon^{k}_{i,t} s^{\text{in}}_{k,j} \\
    & \Delta w^x_{i,j} = \alpha \cdot {\textstyle\sum}_{k=1}^N \delta^{k}_i s^{\text{in}}_{k,j}.
\end{align*}
The equality $\Delta \theta^x = \Delta w^x$  directly follows from $\varepsilon^{k}_{i,1} = \delta^{k}_i$.
\end{proof}

\begin{lemma}
    \label{lem:pcn-varepsilon-iterative-app-rnn}
    Let $M$ be a recurrent PCN trained with Z-IL on a sequential input $S^{in}$ of length $N$. Furthermore, let us assume that $\gamma \,{=}\, 1$ and $\varepsilon^k_{i,0} \,{=}\, 0$ for every $k \in \{1, \dots, N\}$. Then, the prediction error $\varepsilon^{\scriptscriptstyle {k}}_{i,t}$ at $t=1$ (i.e., $\varepsilon^{\scriptscriptstyle {k}}_{i,1}$)  can be derived from the previous recurrent layer.
    Formally:
     \begin{align}
        \varepsilon^{\scriptscriptstyle {k}}_{i,1}= f' ( \mu^{\scriptscriptstyle {k}}_{i,0} ) {\textstyle\sum}_{j=1}^{n^{\scriptscriptstyle {k+1}}} \varepsilon^{\scriptscriptstyle {k+1}}_{j,1} \theta^{\scriptscriptstyle {h}}_{j,i} ,  
    \end{align}
    for $k \in \lbrace 1, \ldots , N-1 \rbrace\,.$
\end{lemma}

\begin{proof}
    Equivalent to the one of Lemma~\ref{lem:pcn-varepsilon-iterative-app}. The only difference is that in Lemma~\ref{lem:pcn-varepsilon-iterative-app} we iterate over the previous layer $l$ at time $t=l$, while here the iterations happen over the previous recurrent layer $k$ at fixed time $t=1$.
\end{proof}

\section{Proof of the Main Theorem}
In this section, we prove the main theorem of our work, which has already been stated in the main body.
\begin{theorem}
Let $(\bar z,y)$ and $(\bar \zeta,y)$  be two points with the same label $y$, and $\mathcal G: \mathbb R^n \rightarrow \mathbb R$ be a function. Assume that the update $\Delta \bar z$ is computed using BP, and the update $\Delta \bar \zeta$ uses Z-IL with $\gamma = 1$. 
Then, if $\bar z = \bar \zeta$, and we consider a levelled computational graph of $\mathcal G$, we have 
\begin{equation}
    \Delta z_i = \Delta \zeta_i,
\end{equation}
for every $i \leq n$.
\end{theorem}
\begin{proof}
As Z-IL acts on the levelled version of $G$, in this proof we consider levelled computational graphs, i.e., graphs where the distance from the top generates a partition of the vertices. We denote $d_i$ the distance  of a vertex $v_i$ to the root vertex $v_{out}$, i.e., $d_i = k$ if $v_i \in S_k$. Furthermore, we denote by $d_{max}$ the maximum distance between the root and any vertex $v_i$, i.e.,  $d_{max} = \max_i d_i$. 

We now divide the proof in two parts, which we call Claim $1$ and Claim $2$. The first part of the proof (i.e., Claim $1$) consists in showing that the errors $\delta_i$ and $\varepsilon_{i,t}$ are equal when $v_i \in S_k$ and $t=k$, which is the time at which the input parameters get updated. Particularly:

\noindent\textbf{Claim $1$:} At any fixed time $t$, we have $\varepsilon_{i,t} = \delta_{i}$ for every $v_i \in S_t$.

\smallskip 
We prove this claim by induction on $d_{max}$. Let us start with the \emph{basic step} $d_{max} = 1$:

We have the output vertex $v_{out}$ and leaf vertices. The value $\mu_{out,t}$ of the output node is given by the elementary function $g_{out}$ defined on all the input variables. Hence, we have 
\begin{equation}
\delta_i = \varepsilon_{i,0} = \mu_{out,t}-y.
\end{equation}
This proves the basic case. Now we move to the \emph{induction step}: let us assume that Claim 1 holds for every computation graph with $d_{max} = m$. 

Let $\mathcal G: \mathbb R^n \rightarrow \mathbb R$ be a function whose computation graph $G(V,E)$ has $d_{max} = m+1$.  For every non-leaf node $v_i$ such that $d_i < m$ and $v_i \in S_t$, we have that $\delta_i = \varepsilon_{i,t}$. Furthermore, note that $\varepsilon_{i,t} = \varepsilon_{i,d_i}$.
\begin{align*}
      \varepsilon_{i,d_i} & = \sum_{j \in P(i)} \varepsilon_{j,d_i-1} \frac{\partial \mu_{j,d_i}}{\partial x_{i,0}} & \mbox{\ by Lemma~\ref{lem:1},}\\
     \delta_i & = \sum_{j \in P(i)} \delta_j \frac{\partial \mu_{j}}{\partial \mu_{i}} & \mbox{\ by Eq.~\eqref{eq:delta_err}.}
\end{align*}
The two quantities above are equal. This follows from the induction step, which gives $\varepsilon_{j,d_i-1} = \delta_i$ and from the condition that states that $\mu_{i,t} = x_{i,0}$ for $d<d_i$. This concludes the proof of Claim $1$.

\smallskip 
\noindent\textbf{Claim $2$:} We have $\Delta z_i = \Delta \zeta_i$ for every $i \leq n$.

\smallskip 
Eqs.~\eqref{eq:deltaz} and \eqref{eq:deltazeta} state the following:
\begin{align*}
    \Delta z_i & =  \alpha \cdot \sum_{j \in P(i)} \delta_j \frac{\partial \mu_{j}}{\partial z_i},\\
    \Delta \zeta_i & =  \alpha \sum_{j \in P(i)} \varepsilon_{j,t}\frac{\partial \mu_{j,t}}{\partial \zeta_i}.
\end{align*}
The update of every input parameter $\zeta_i$ in Z-IL happens at $t = d_i$. Claim $1$ shows that, at that specific time, we have $\delta_i = \varepsilon_{i,t}$, while Lemma~\ref{lem:2} states that $\mu_{i,t} = \mu_{i,0}$ for every $t \leq d_i$. The proof of the claim, and hence, the whole theorem, follows from $\zeta_i = z_i$ for every $i\leq n$.
\end{proof}

\begin{lemma}\label{lem:1}
Let $\bar \zeta$ be an input of a continuous and differentiable function $\mathcal G: \mathbb R^n \rightarrow \mathbb R$ with computational graph $G(V,E)$, and also assume that the update $\Delta \bar \zeta$ using Z-IL with the partition of $V$ described by Eq.~\eqref{eq:sub}, we then have $\mu_{i,t} = \mu_{i,0}$ and $\varepsilon_{i,t} = 0$ for every $t \leq d_i$.
\end{lemma}

\begin{proof}
This directly follows from the fact that we are applying Z-IL on a levelled graph. In fact, the value $\mu_{i,d_i}$ of every vertex $v_i$ differs from its initial state $\mu_{i,0}$ only if the node values $\{x_{j,t}\}_{j \in C(i)}$ of the children vertices have changed in the time interval $[0,d_i]$. This may only happen if we have $d_j < d_i$ for one of the vertices $\{ v_{j} \}_{j \in C(i)}$. But this is impossible, as the distance from the top $d_i$ of a parent node is always strictly smaller than the one of any of its children nodes in a levelled graph. 
\end{proof}

\begin{lemma}\label{lem:2}
    The prediction error in Z-IL at $t=d_i$, i.e., $\varepsilon_{i,t}$, can be derived from itself at previous inference moments. Formally,
     \begin{align}
        \varepsilon_{i,d_i}= \gamma \sum_{j \in P(i)} \varepsilon_{j,d_i-1} \frac{\partial \mu_{j,d_i}}{\partial x_{i,0}}. 
    \end{align}
\end{lemma}

\begin{proof}
    Let us write  $\varepsilon_{i,t}$ as a function of $\varepsilon_{i,t-1}$:
    \begin{equation}
        \varepsilon_{i,t} = \varepsilon_{i,t-1} + {(\Delta{x}_{i,t-1} - \Delta{\mu}_{i,t-1})\,,}
    \end{equation}
    where $\Delta{\mu}_{i,t-1}=\mu_{i,t}-\mu_{i,t-1}$.
    Then, we expand $\varepsilon_{i,d_i}$ with the above equation and simplify it with Lemma~\ref{lem:1}, i.e., $\varepsilon_{i,d_i-1}=0$ and $\Delta{\mu}_{i,t<d_i-1}=0$:
    \begin{align}
        \varepsilon_{i,d_i} 
        = \varepsilon_{i,d_i-1} + {(\Delta{x}_{i,d_i-1} - \Delta{\mu}_{i,d_i-1})}
        ={\Delta{x}_{i,d_i-1}}. \label{eq:varepsilon-dotx-cg}
    \end{align}
    We further investigate $\Delta{x}_{i,d_i-1}$ expanded with the inference dynamic Eq.~\eqref{eq:deltax-cg} and simplify it with Lemma~\ref{lem:1}, i.e., $\varepsilon_{i,t<d_i}=0$,
    \begin{align}
            \Delta{x}_{i,d_i-1}
            &= \gamma( \varepsilon_{i,d_i-1} + {\textstyle\sum}_{j \in P(i)} \varepsilon_{j,d_i-1} \frac{\partial \mu_j}{\partial x_{i,d_i-1}}) \\
            &= \gamma {\textstyle\sum}_{j \in P(i)} \varepsilon_{j,d_i-1} \frac{\partial \mu_j}{\partial x_{i,d_i-1}}.
            \label{eq:dotx-iterative}
    \end{align}
    Putting Eq.~\eqref{eq:dotx-iterative} into Eq.~\eqref{eq:varepsilon-dotx-cg}, we obtain:
    \begin{align}
        \varepsilon_{i,d_i} 
        = & \ \gamma \sum_{j \in P(i)} \varepsilon_{j,d_i-1} \frac{\partial \mu_{j,d_i}}{\partial x_{i,d_i-1}} \\
        = & \ \gamma \sum_{j \in P(i)} \varepsilon_{j,d_i-1} \frac{\partial \mu_{j,d_i}}{\partial x_{i,0}}.
    \end{align}
    With Lemma~\ref{lem:1}, $x_{i,d_i-1}$ can be replaced with $x_{i,0}$.
\end{proof}

\end{document}